\documentclass{article}
\usepackage[preprint]{neurips_2026}

\usepackage[utf8]{inputenc}
\usepackage[T1]{fontenc}
\usepackage{amsmath,amssymb,amsfonts}
\usepackage{amsthm}
\usepackage{booktabs}
\usepackage{graphicx}
\usepackage{hyperref}
\usepackage[ruled,vlined]{algorithm2e}
\usepackage{multirow}
\usepackage{xcolor}
\usepackage{subcaption}
\makeatletter
\providecommand\FloatBarrier{%
  \par
  \begingroup
    \let\@elt\relax
    \xdef\@freelist{\@freelist\@midlist}\global\let\@midlist\@empty
    \@floatplacement\@dblfloatplacement
    \@makefcolumn\@deferlist
    \@whilesw\if@fcolmade\fi{\@outputpage \@makefcolumn\@deferlist}%
  \endgroup}
\makeatother
\newtheorem{proposition}{Proposition}
\newtheorem{corollary}[proposition]{Corollary}

\title{
When Graph Language Models Go Beyond Memorization
}

\author{%
  Masatsugu Yamada \\
  National Institute of Informatics \\
  Tokyo, Japan \\
  \texttt{masatsugu-yamada@nii.ac.jp} \\
  \And
  Mahito Sugiyama \\
  National Institute of Informatics, Tokyo, Japan \\
  SOKENDAI, Kanagawa, Japan \\
  \texttt{mahito@nii.ac.jp} \\
}

\date{}

\begin{document}
\maketitle

\begin{abstract}
It remains unclear whether graph language models learn structural regularities or merely memorize training graphs; this cannot be resolved by current aggregate
fidelity metrics alone.
We develop a calibrated diagnostic protocol that combines frequent subgraph mining,
a graph-level bootstrap baseline, and three-level frequency stratification
to disentangle memorization from structural alignment.
Using this framework, we show that graph language models can acquire structural regularities beyond memorization at scale, primarily in the high-frequency regime. This is supported by the following empirical evidence:
On five TU benchmarks, LLaMA-style graph language models reach high subgraph-rank
correlation, yet their alignment is matched or exceeded by the memorization bootstrap in most cases.
At small scale, under our bootstrap diagnostic, fidelity is largely indistinguishable
from verbatim recall.
In contrast, at large scale with 3.75M graphs, verbatim memorization drops
sharply while rank correlation remains near ceiling.
Crucially, in a separate fixed-subsample analysis,
frequent subgraph mining restricted to the novel-only subset
closely tracks the
corresponding all-generation Spearman correlation,
providing evidence that the alignment is not driven solely by verbatim recall.
%
Across all scales,
high-frequency patterns are well reproduced, while rare patterns remain poorly covered, and this
deficit narrows only marginally as capacity increases.
We observe the same scale-dependent crossover under two distinct graph serializations
(canonical DFS code and action sequences), providing evidence of robustness in our analysis.
\end{abstract}

\section{Introduction}
\label{sec:intro}

Graph generation is a fundamental problem in machine learning with
applications across drug discovery, material design, and social network
analysis~\citep{you2018graphrnn,li2018dgmg,simonovsky2018graphvae,vignac2023digress,jo2022gdss}.
A growing line of work serializes graphs into token sequences and trains
\emph{graph language models}, LLM-style autoregressive
decoders, to generate them, reporting strong fidelity on standard
benchmarks~\citep{goyal2020graphgen, bagal2022molgpt, chen2025flatten}. 
Whether these models fundamentally learn structural
regularities of graphs, or merely \emph{memorize} training data and
emit graphs from a near-empirical distribution, remains an open
question that current aggregate fidelity metrics
(degree/clustering/orbit MMD, GIN-based MMD/FID) cannot settle~\citep{you2018graphrnn, thompson2022evaluation}: as we formalize in Corollary~\ref{cor:non-identifiability}, agreement on any single
bounded subgraph statistic is consistent with both regimes.

In this paper, we address this issue with a \emph{calibrated diagnostic protocol} that focuses on \emph{subgraph} statistics in graph databases. This protocol
combines frequent subgraph mining (gSpan)~\cite{yan2002gspan}, a graph-level bootstrap
baseline that resamples directly from the training corpus, and 
three levels of subgraph frequency stratification, Head, Torso, and Tail, of the support distribution.
The bootstrap supplies the missing reference: a model is taken to go
beyond memorization only when its support statistics depart from the
distribution induced by resampling training graphs, with frequency
stratification then exposing where the gap concentrates.

We apply this protocol to investigate a core hypothesis: 
can graph language models act as \emph{implicit neural graph miners}? 
While classical algorithms such as gSpan explicitly search combinatorial spaces to extract frequent subgraphs, autoregressive modeling of graph serializations may allow LLMs to internalize these 
same structural regularities implicitly without ever performing 
explicit pattern enumeration. Whether such implicit mining is possible, and, if so, to what extent, or instead confounded by verbatim recall of training graphs, 
is the central question that our calibrated protocol is designed to answer. 

The key idea is to disentangle \emph{implicit structural learning} from \emph{memorization-driven alignment} along two axes simultaneously: a \emph{memorization} axis (verbatim recall vs.\ novel generation) and an \emph{alignment} axis (distributional fidelity to training subgraph statistics). 
The bootstrap baseline serves as the critical reference that makes this disentanglement 
operational. Furthermore, crossing these axes with pattern frequency 
stratification allows us to ask \emph{where} across the support 
distribution this implicit mining succeeds or fails---a question 
that aggregate metrics, by construction, cannot pose.

Our contributions are summarized as follows:
\begin{itemize}
  \item \textbf{Calibrated diagnostic protocol.}
    We introduce an evaluation protocol combining frequent subgraph mining with a memorization bootstrap. This protocol overcomes the limitations of aggregate fidelity metrics by disentangling true structural learning from training data memorization across different subgraph frequencies (\S\ref{sec:setup-theory}, \S\ref{sec:method}).

  \item \textbf{Scale-dependent structural learning.}
    We reveal that graph language models undergo a phase transition depending on scale. While models on small datasets operate in a \emph{memorization-dominated} regime, models trained at scale enter a \emph{decoupled} regime, successfully preserving subgraph statistics even for novel, unmemorized graphs (\S\ref{sec:results}).

  \item \textbf{Persistent rare-pattern deficit.}
    Robust across different graph serializations and model capacities, our stratified diagnostic exposes a fundamental limitation: while models capture frequent substructures well, they consistently fail to learn the distribution of rare (Tail) patterns. This highlights a critical gap in current autoregressive graph generation (\S\ref{sec:results}).
\end{itemize}
\section{Related Work}
\label{sec:related}

\paragraph{Graph generation.}
Autoregressive generators produces graphs sequentially:
GraphRNN~\citep{you2018graphrnn} produces adjacency rows via an RNN,
GraphGen~\citep{goyal2020graphgen} operates on minimum DFS codes,
DGMG~\citep{li2018dgmg} uses a node/edge decision process, and recent
transformer-based serialized generators such as
AutoGraph~\citep{chen2025flatten} decode graphs from a flattened token
stream. Diffusion (DiGress~\citep{vignac2023digress},
GDSS~\citep{jo2022gdss}, EDP-GNN~\citep{niu2020edpgnn}), VAE
(GraphVAE~\citep{simonovsky2018graphvae}), and molecular generators
(GCPN~\citep{you2018gcpn}, GraphAF~\citep{shi2020graphaf}) provide non-autoregressive alternatives; 
we use DiGress as our same-corpus diffusion comparator because it is a recent labeled-graph diffusion model with public training code that scales to both TU and PCQM4Mv2 under our pipeline.
GraphGPT~\citep{zhao2024graphgpt} is a self-supervised graph foundation model for representation learning rather than generation. Our work is closest to GraphGen in using DFS codes, but we employ Transformer LLMs and focus on evaluation methodology rather than architecture.

\paragraph{Graph language models.}
We use \emph{graph language models} for LLM-style autoregressive
decoders trained on serialized graphs---the family our diagnostic
targets. Existing instances include
GraphGen~\citep{goyal2020graphgen} (RNN over DFS codes),
MolGPT~\citep{bagal2022molgpt} (SMILES-based molecular generation),
AutoGraph~\citep{chen2025flatten} (Transformer over flattened token
streams), and the GraphGPT line~\citep{zhao2024graphgpt} (graph
foundation models for representation, not generation). Adjacent work
targets reasoning over text-encoded
graphs~\citep{fatemi2024talklike,sanford2024transformer} or molecular
property metrics (validity / novelty / drug-likeness). None of these
lines probes whether the generated set reproduces local
\emph{subgraph} statistics of the training distribution at the pattern
level, nor calibrates such alignment against a memorization
reference. We characterize graph language models along these axes.

\paragraph{Memorization-aware evaluation and subgraph mining.}
Memorization has been extensively analyzed in language
models~\citep{carlini2021extracting,carlini2023quantifying} and image
diffusion~\citep{somepalli2023diffusion}: larger models memorize
more~\citep{carlini2023quantifying}, repeats concentrate
memorization~\citep{carlini2023quantifying}, and training data can be
extracted~\citep{carlini2021extracting}. These lines, however,
contribute \emph{detectors} of verbatim training examples; none
provides a distributional reference against which aggregate fidelity
metrics can be calibrated. The graph-generation literature is even
sparser in this regard, with prior memorization analysis limited to
whole-graph novelty rates~\citep{you2018graphrnn}. We close this gap
on the substructure side, re-purposing gSpan~\citep{yan2002gspan},
which enumerates all frequent connected subgraph patterns via a
canonical DFS code form with a minimum-support threshold, as the
evaluation backbone. We provide the first
label-aware subgraph-statistic evaluation calibrated against a
graph-level memorization bootstrap.

\section{A Distributional View of Sequence-Based Graph Generation}
\label{sec:setup-theory}

We frame sequence-based graph generation through the pushforward of a
graph distribution along a serialization map. This view yields two
elementary propositions that show why high distributional alignment in
the sequence space, and even in subgraph statistics derived from graphs, 
can be confounded with memorization, motivating the calibrated protocol of
\S\ref{sec:method}.

\noindent\textbf{Notation.}
Let $\mathcal{G}$ be a space of (vertex- and edge-labeled) graphs,
$\mathcal{X}$ a space of token sequences, and
$\bar{\mathcal{G}}=\mathcal{G}\cup\{\bot\}$ the graph space augmented
with a sink state for malformed parses. A serialization map
$\phi: \mathcal{G} \to \mathcal{X}$ (canonical DFS code, DGMG action
sequence, etc.) and a measurable recovery map
$\psi: \mathcal{X} \to \bar{\mathcal{G}}$ satisfy $\psi(\phi(G)) = G$.
Let $P_G$ be the true graph distribution. We denote by $P_X = \phi_\# P_G$ 
its pushforward measure along $\phi$; 
intuitively, $P_X$ is the distribution 
of the serialized sequences $\phi(G)$ when $G \sim P_G$ 
(formally, $P_X(A) = P_G(\phi^{-1}(A))$ for any measurable set $A \subseteq \mathcal{X}$). 
Similarly, a generative model defines a sequence distribution $Q_X$, 
which induces a graph distribution $Q_G = \psi_\# Q_X$ on $\bar{\mathcal{G}}$ 
via the recovery map. 
We extend $P_G$ to $\bar{\mathcal{G}}$ by assigning zero mass to $\bot$ 
when comparing distributions across the two spaces.

\begin{proposition}[Recovery is non-expansive in total variation (TV) distance]
\label{prop:tv}
For any sequence distributions $Q_X, P_X$ on $\mathcal{X}$, the following inequality for the total variation distance $d_{\mathrm{TV}}$ holds:
\[
  d_{\mathrm{TV}}(\psi_\# Q_X, \psi_\# P_X)
  \;\le\;
  d_{\mathrm{TV}}(Q_X, P_X).
\]
\end{proposition}

\begin{proposition}[Bounded graph statistics inherit closeness]
\label{prop:stat}
For any bounded function $f: \bar{\mathcal{G}} \to \mathbb{R}$,
\[
  \big|\mathbb{E}_{Q_G} f - \mathbb{E}_{P_G} f\big|
  \;\le\;
  2\|f\|_\infty\, d_{\mathrm{TV}}(Q_G, P_G)
  \;\le\;
  2\|f\|_\infty\, d_{\mathrm{TV}}(Q_X, P_X).
\]
\end{proposition}

Choosing $f$ as the indicator (or bounded support) of a fixed subgraph
$H$, extended by $f(\bot)=0$, shows that closeness in the sequence
distribution implies closeness in subgraph statistics. Both
propositions follow from the data-processing inequality and the dual
characterization of total variation; short proofs are in
Appendix~\ref{app:proofs}.

\begin{corollary}[Non-identifiability of bounded subgraph statistics]
\label{cor:non-identifiability}
Let $\hat P_G^{\mathrm{boot}}$ denote the empirical (graph-level
bootstrap) distribution on the training corpus
$\{G_1, \dots, G_n\} \subseteq \mathcal{G}$. For any bounded function
$f:\bar{\mathcal{G}} \to \mathbb{R}$ and any $\varepsilon > 0$, there
exist sequence-level model distributions $Q_X^{(\mathrm{mem})}$ and
$Q_X^{(\mathrm{gen})}$ such that
$\bigl|\mathbb{E}_{Q_G^{(\mathrm{mem})}} f - \mathbb{E}_{Q_G^{(\mathrm{gen})}} f\bigr| < \varepsilon$,
where $Q_G^{(\mathrm{mem})}$ concentrates on $\{G_1, \dots, G_n\}$
(verbatim recall) and $Q_G^{(\mathrm{gen})}$ places positive mass
outside the training corpus.
\end{corollary}

Corollary~\ref{cor:non-identifiability} is the operational reading of
Propositions~\ref{prop:tv} and~\ref{prop:stat}: agreement on any single
bounded subgraph statistic, including support-based ranking and
divergence summaries, cannot separate verbatim recall from
distribution-level reproduction without a memorization reference. A
generator should be considered to go \emph{beyond} memorization only
when its statistic lies outside the bootstrap reference distribution
sampled from $\hat P_G^{\mathrm{boot}}$. This pins down the central
methodological point of the paper: \emph{aggregate distributional
metrics must be calibrated against memorization to be interpretable},
and motivates the two empirical axes, memorization rate and
distributional alignment, that together place every generator on the
\emph{memorization--alignment spectrum} we adopt throughout the paper.

\noindent\textbf{Bootstrap baseline as graph-level resampling.}
Given $n$ training graphs $G_1, \dots, G_n$, the graph-level bootstrap
estimator is
$\hat P_G^{\mathrm{boot}}(G) = \frac{1}{n}\sum_{i=1}^{n} \mathbf{1}[G = G_i]$,
the empirical distribution. It places mass only on training graphs:
$\hat P_G^{\mathrm{boot}}(G) = 0$ for any
$G \notin \{G_1, \dots, G_n\}$. By contrast, an autoregressive model
factorizes $Q_G(G) = Q_X(\phi(G)) = \prod_{t} Q(x_t \mid x_{<t})$,
which can place positive mass on unseen graphs. The gap between
\emph{memorization-dominated} regimes ($Q_G \approx
\hat P_G^{\mathrm{boot}}$) and \emph{distributional-generalization}
regimes ($Q_G$ deviates from $\hat P_G^{\mathrm{boot}}$ in a
structurally meaningful direction) is what our calibrated protocol
aims to expose.

\section{Calibrated Evaluation Framework}
\label{sec:method}

We operationalize Corollary~\ref{cor:non-identifiability}
as a concrete evaluation framework in this section. We first specify the graph
serializations and training setup (\S\ref{sec:method-serial}), then
define the subgraph statistics and whole-graph exact-match rate that
serve as the basic observables of our diagnostic
(\S\ref{sec:method-gspan}). On top of these quantities, we introduce two primary diagnostics: a graph-level bootstrap calibration that  provides an empirical memorization reference (\S\ref{sec:method-bootstrap}), 
and a frequency-stratified analysis that identifies where alignment breaks down across the support distribution (\S\ref{sec:method-strata}). 
Together with the sequence-space hit/miss analysis reported in Appendix~\ref{app:nll},
these diagnostics place each model on the memorization--alignment
plane introduced in \S\ref{sec:setup-theory}, rather than collapsing
its behavior to a single aggregate fidelity score.

\subsection{Graph Serialization and Training}
\label{sec:method-serial}

We consider two well-established methods for serialization $\phi : \mathcal{G} \to \mathcal{X}$.
\textbf{Canonical DFS code}, developed to efficiently enumerate subgraphs in gSpan~\citep{yan2002gspan}, runs a depth-first traversal of
$G = (V, E, \ell_V, \ell_E)$ to emit edge tuples
$t = \langle i, j, \ell_i, \ell_j, \ell_{ij}\rangle$ ($i < j$
forward, $i > j$ backward), totally orders tuple sequences
lexicographically, and selects the minimum---the canonical DFS code
$C_{\min}(G)$, unique up to graph isomorphism---which we
serialize tuple-by-tuple with one token per tuple
(rightmost-path construction in Algorithm~\ref{alg:canonical-dfs},
Appendix~\ref{app:algorithms}). \textbf{DGMG action
sequences}~\citep{li2018dgmg} encode $G$ as a sequential decision process
(\texttt{ADD\_NODE}, \texttt{ADD\_EDGE}, \texttt{DEST}, \texttt{END\_EDGE},
\texttt{END\_NODE}); we canonicalize by adding nodes in sorted-label
order so isomorphic graphs share the same action string.
We train decoder-only LLaMA causal language models~\citep{touvron2023llama} 
from scratch on the serialized sequences with the standard language 
modeling objective, minimizing the per-token negative log-likelihood (NLL).
TU-benchmark training uses a fixed 80M-token budget; the PCQM4Mv2 schedule is in
\S\ref{sec:setup}.

\subsection{Subgraph Statistics via Frequent Pattern Mining}
\label{sec:method-gspan}

Let $\mathbf{D}_{\mathrm{tr}}$ and $\mathbf{D}_{\mathrm{gen}}$ denote
the training and generated graph databases (finite samples from $P_G$
and $Q_G$, respectively, in the notation of
\S\ref{sec:setup-theory}). We write $g \sqsubseteq G$ when $g$ is a
(vertex- and edge-labeled) subgraph of $G$, and define the
\emph{support} of a pattern $g$ in a database $\mathbf{D}$ as
\[
  \mathrm{supp}_{\mathbf{D}}(g) = \frac{1}{|\mathbf{D}|}
    \sum_{G \in \mathbf{D}} \mathbf{1}\bigl[g \sqsubseteq G\bigr].
\]
gSpan~\citep{yan2002gspan} with minimum-support ratio $\sigma = 0.1$ enumerates all subgraphs whose supports are larger than $\sigma$, yielding the
frequent-pattern sets $\mathcal{P}_{\mathrm{tr}}$ and
$\mathcal{P}_{\mathrm{gen}}$. We use their union
$\mathcal{U} = \mathcal{P}_{\mathrm{tr}} \cup \mathcal{P}_{\mathrm{gen}}$ and intersection
$\mathcal{C} = \mathcal{P}_{\mathrm{tr}} \cap \mathcal{P}_{\mathrm{gen}}$,
and define support and normalized pattern probability for each
$g \in \mathcal{U}$ as
\[
  s_{\mathrm{tr}}(g) = \mathrm{supp}_{\mathbf{D}_{\mathrm{tr}}}(g),
  \quad
  p_{\mathrm{tr}}(g) = \frac{s_{\mathrm{tr}}(g)}
    {\sum_{h \in \mathcal{U}} s_{\mathrm{tr}}(h)},
\]
and analogously $s_{\mathrm{gen}}, p_{\mathrm{gen}}$ on
$\mathbf{D}_{\mathrm{gen}}$.

\noindent\textbf{Distributional metrics.}
We quantify the alignment of $\mathcal{P}_{\mathrm{gen}}$ to $\mathcal{P}_{\mathrm{tr}}$ using four scalar summaries:
\emph{Spearman's $\rho$} is the rank correlation between
$s_{\mathrm{tr}}$ and $s_{\mathrm{gen}}$ restricted to the
intersection $\mathcal{C}$, capturing preservation of the relative
frequency ordering. \emph{Jensen--Shannon divergence} (JSD) is
$\mathrm{JSD}(p_{\mathrm{tr}}, p_{\mathrm{gen}}) =
\mathrm{KL}(p_{\mathrm{tr}} \| m)/2 +
\mathrm{KL}(p_{\mathrm{gen}} \| m)/2$ with mid-point
$m = (p_{\mathrm{tr}} + p_{\mathrm{gen}})/2$. The training-side
\emph{missing mass}
$\mathrm{MM} = \sum_{g \in \mathcal{P}_{\mathrm{tr}}\setminus
  \mathcal{P}_{\mathrm{gen}}} p_{\mathrm{tr}}(g)$
records the probability mass on patterns absent from the generated
set; the generated-side \emph{novel mass}
$\mathrm{NM} = \sum_{g \in \mathcal{P}_{\mathrm{gen}}\setminus
  \mathcal{P}_{\mathrm{tr}}} p_{\mathrm{gen}}(g)$
is its complement on patterns absent from training.

\noindent\textbf{Whole-graph memorization.}
At the whole-graph level we measure verbatim recall through canonical-DFS-code
equality: writing $\mathrm{code}(\mathbf{D}) = \{\mathrm{code}(G) :
G \in \mathbf{D}\}$,
\begin{equation}
\label{eq:em}
  \mathrm{EM} = \frac{|\{G \in \mathbf{D}_{\mathrm{gen}} :
    \mathrm{code}(G) \in \mathrm{code}(\mathbf{D}_{\mathrm{tr}})\}|}
    {|\mathbf{D}_{\mathrm{gen}}|},
  \qquad
  \mathrm{Novelty} = 1 - \mathrm{EM}.
\end{equation}
The exact-match rate $\mathrm{EM}$, reported as ``precision'' in
result tables for backwards compatibility, is the empirical
counterpart of the bootstrap estimator $\hat P_G^{\mathrm{boot}}$ of
\S\ref{sec:setup-theory}: when the model collapses onto memorized
training graphs, $Q_G \approx \hat P_G^{\mathrm{boot}}$ and
$\mathrm{EM} \to 1$.

\noindent\textbf{From Propositions to Protocol.}
We introduce the two diagnostics in the following subsections, which are not independent benchmarks but
operational projections of
Corollary~\ref{cor:non-identifiability}. The graph-level bootstrap
(\S\ref{sec:method-bootstrap}) realizes the memorization reference
$\hat P_G^{\mathrm{boot}}$ empirically, so that comparing a model's
$\rho$ or $\mathrm{JSD}$ against the bootstrap distribution tests the
exact ambiguity the Corollary forbids one to ignore. Frequency
stratification (\S\ref{sec:method-strata}) sharpens
Proposition~\ref{prop:stat} by restricting the test function $f$ to a
support sub-band, exposing where aggregate $\rho$ hides a structural
deficit. As a complementary sequence-space check, per-sequence NLL hit/miss (Appendix~\ref{app:nll}) 
is the sequence-space counterpart of Proposition~\ref{prop:tv}: 
it asks whether the TV gap concentrates on the precise sequences the model later reproduces. 
Reading the three diagnostics jointly returns a verdict on the two-dimensional
memorization--alignment plane rather than another scalar fidelity
number.

\subsection{Bootstrap Calibration}
\label{sec:method-bootstrap}

A high Spearman $\rho$ between training and generated sets may simply
reflect the fact that most generated graphs are drawn from the
training distribution via memorization. To calibrate this, we
construct a bootstrap baseline that samples directly from the
graph-level estimator $\hat P_G^{\mathrm{boot}}$ of
\S\ref{sec:setup-theory} as an empirical memorization reference: it
estimates the distributional alignment expected when generated graphs
are drawn directly from the training corpus.
Concretely, for $R = 10$ repeats we draw
$|\mathbf{D}_{\mathrm{gen}}|$ training graphs with replacement, run
gSpan on each resample, and compute $\rho^{(r)}, \mathrm{JSD}^{(r)},
\mathrm{MM}^{(r)}$ against $\mathcal{P}_{\mathrm{tr}}$ (full pseudo code
in Algorithm~\ref{alg:bootstrap-calibration},
Appendix~\ref{app:algorithms}). We report the model's percentile rank
within the resulting bootstrap distribution and the calibrated
$z$-score
$z = (\rho_{\mathrm{model}} - \mu_{\mathrm{boot}}) / \sigma_{\mathrm{boot}}$.

\subsection{Frequency-Stratified Analysis}
\label{sec:method-strata}

Aggregate metrics over $\mathcal{U}$ obscure how alignment varies
with pattern frequency, which is exactly the regime where
memorization-driven and structurally-learned models diverge most. We
therefore sort the training patterns $\mathcal{P}_{\mathrm{tr}}$ by
support $s_{\mathrm{tr}}$ and partition them into a \emph{Head}
(top $10\%$, high-frequency motifs that dominate the distribution),
a \emph{Torso} (middle $80\%$, moderately frequent patterns), and a
\emph{Tail} (bottom $10\%$, rare patterns near the support
threshold). Within each stratum we recompute Spearman~$\rho$,
$\mathrm{JSD}$, and $\mathrm{MM}$ on the corresponding restriction
of $\mathcal{U}$, so that any frequency-dependent generalization
gap---typically a Head/Torso--Tail split---becomes directly
visible rather than averaged out.

As a complementary per-sequence check, we compute average per-token
NLL under the model for every unique training sequence and compare
\emph{hits} (reproduced in generation) against \emph{misses} (not
reproduced) via the Mann--Whitney $U$ test (full results in
Appendix~\ref{app:nll}).

\section{Experimental Setup}
\label{sec:setup}

We evaluate on five TU benchmarks~\citep{morris2020tudataset} and the
PCQM4Mv2 molecular corpus from OGB-LSC~\citep{hu2021ogblsc}
(per-dataset graph statistics in
Appendix~\ref{app:serialization}, Table~\ref{tab:datasets});
PCQM4Mv2 contributes $3{,}746{,}620$
graphs with a canonical-DFS training split of $3{,}371{,}958$
sequences ($3{,}104{,}677$ unique after deduplication). Models are
LLaMA~\citep{touvron2023llama} \textsc{small} (132M parameters; 12 layers, hidden 768, 12 heads, FFN 3072) trained from scratch with causal LM, TRL SFTTrainer
+ DeepSpeed ZeRO-1~\citep{rajbhandari2020zero}, AdamW~\citep{loshchilov2019decoupled} ($\text{lr}=10^{-5}$), cosine schedule,
bf16, an 80M-token budget, and 20 evenly spaced checkpoints; for each
TU dataset we train one DFS-canonical and one DGMG model. PCQM4Mv2
additionally sweeps \textsc{tiny}/\textsc{medium}/\textsc{large}
variants (10.5M--451M parameters; PCQM4Mv2 scaling results below). For each
model we sample $1024$ sequences under six decoding configurations
(default \texttt{both\_default}: $T\!=\!1.0$, top-$p\!=\!0.95$,
top-$k\!=\!50$; full sweep in Appendix~\ref{app:decode-sweep}),
parse them into graphs (regex for DFS, action parsing for DGMG), and
run the evaluation pipeline of \S\ref{sec:method}: gSpan
($\sigma=0.1$), distributional metrics, bootstrap baselines, subgraph pattern strata,
NLL hit/miss, and Weisfeiler--Lehman (WL) kernel MMD~\citep{shervashidze2011weisfeiler}. 
Serialized-sequence length and vocabulary statistics are in
Appendix~\ref{app:serialization}~(Table~\ref{tab:seq-stats-app}).

\section{Results}
\label{sec:results}

\noindent\textbf{Diagnostic expectations.}
The framework of \S\ref{sec:method} commits to four concrete
expectations; the TU/PCQM regimes below are read against them rather
than reverse-engineered from the tables.

\noindent\textbf{(E1)} Under Corollary~\ref{cor:non-identifiability},
a memorization-dominated regime places $\rho$ inside the graph-level
bootstrap reference, making subgraph-level aggregate fidelity
uninformative (calibrated baseline comparison below).

\noindent\textbf{(E2)} A regime that goes beyond memorization retains
$\rho$ when gSpan is restricted to the novel-only subset
$\mathbf{D}_{\mathrm{gen}}\setminus \mathbf{D}_{\mathrm{tr}}$, so the
subgraph-level alignment is not carried by whole-graph recall alone
(PCQM4Mv2 scaling below).

\noindent\textbf{(E3)} Frequency stratification exposes, rather than
averages over, support deficits hidden by aggregate $\rho$
(frequency-stratified analysis below).

\noindent\textbf{(E4)} If the capacity is structural, the same
crossover is reproduced under a distinct serialization rather than
tied to a specific tokenization (PCQM4Mv2 scaling below;
Appendix~\ref{app:dgmg-strata}).

\noindent\textbf{Overall distribution alignment.}
Table~\ref{tab:main-results} presents the primary generation metrics
under \texttt{both\_default} decoding for all datasets in canonical
DFS code format.

\begin{table}[t]
\centering
\caption{Generation metrics under \texttt{both\_default} decoding
  (DFS canonical, single training seed = 42; ENZYMES values are means
  over $5$ decoding seeds under the no-cap gBolt rerun, see
  Limitations~(g)).
  Whole-graph: Unique,
  Precision (= EM, Eq.~\eqref{eq:em}), Novelty.
  Subgraph-level: $|\mathcal{P}_{\mathrm{tr}}|$ and
  $|\mathcal{P}_{\mathrm{gen}}|$ are gSpan subgraph pattern set sizes,
  $\rho$ Spearman rank correlation, JSD Jensen--Shannon divergence,
  Missing the training-side missing mass.}
\label{tab:main-results}
\small
\setlength{\tabcolsep}{4pt}
\begin{tabular}{lcccrrccc}
\toprule
Dataset & Unique & Precision & Novelty & $|\mathcal{P}_{\mathrm{tr}}|$ & $|\mathcal{P}_{\mathrm{gen}}|$ & $\rho$ & JSD & Missing \\
\midrule
MUTAG    & 0.199 & 0.828 & 0.172 & 634       & 714       & 0.976 & 0.035 & 0.014 \\
PTC\_MR  & 0.285 & 0.983 & 0.017 & 127       & 148       & 0.950 & 0.063 & 0.007 \\
ENZYMES  & 0.418 & 1.000 & 0.000 & 95{,}419  & 101{,}710 & 0.987 & 0.014 & 0.007 \\
PROTEINS & 0.591 & 0.891 & 0.109 & 70{,}257  & 75{,}859  & 0.988 & 0.022 & 0.005 \\
NCI1     & 0.862 & 1.000 & 0.000 & 818       & 884       & 0.990 & 0.010 & 0.006 \\
\bottomrule
\end{tabular}
\end{table}

Across all datasets, Spearman correlations are remarkably high
($\rho = 0.950$--$0.990$), indicating that the relative ordering of
frequent subgraph patterns is well-preserved in the generated sets.
Missing mass is uniformly low ($<$1.5\%), suggesting that most training patterns are reproduced.
However, two observations complicate this optimistic subgraph-level reading.
First, at the \emph{whole-graph} level, verbatim memorization rates are high: 82.8\%--100\% of
generated sequences are exact matches of training graphs 
(precision in Table~\ref{tab:main-results}, $\mathrm{EM}$ of Eq.~\eqref{eq:em}).
NCI1 achieves $\rho = 0.990$ with 100\% precision (zero novelty),
meaning every generated graph was already present in training corpus.
Second, generated-set \emph{whole-graph} uniqueness rates inversely correlate with \emph{whole-graph} memorization:
MUTAG generates only 19.9\% unique sequences (heavy duplication),
while NCI1 achieves 86.2\%, reflecting its larger training set.

Standard label-free MMD metrics give the same optimistic reading
($<\!3\!\times\!10^{-3}$ for every TU benchmark; full breakdown in
Appendix~\ref{app:mmd}, including a cross-model comparison that
shows label-free MMD hiding the labeled-structure failure of
unlabeled baselines that WL exposes), but
Cor~\ref{cor:non-identifiability} shows that such aggregate
agreement is compatible with both memorization and structural
learning, motivating the bootstrap calibration below.
Distributional metrics are also strongly modulated by decoding
configuration---inter-config $\rho$ spread can exceed $0.3$, large
enough to shift the bootstrap-calibrated verdict (full sweep in
Appendices~\ref{app:decode-sweep},~\ref{app:factor-decomp}); the
main text reports \texttt{both\_default}.

\noindent\textbf{Calibrated baseline comparison.}
Table~\ref{tab:calibrated} compares model metrics against bootstrap
baselines, which represent the distributional alignment achievable by
simply resampling from the training set.

\begin{table}[t]
\centering
\caption{Model vs.\ bootstrap baseline (single training seed = 42;
  ENZYMES values are aggregated over $5$ decoding seeds $\times$ $10$
  bootstrap repeats under the no-cap gBolt rerun to match the
  no-cap audit in Limitations~(g)).
  $\Delta\rho$: model Spearman minus bootstrap median. $z$:
  standardized score. Pctile: percentile rank of model within
  bootstrap distribution. Values below the bootstrap median are
  highlighted in \textcolor{red}{red}.}
\label{tab:calibrated}
\footnotesize
\begin{tabular}{lcccccc}
\toprule
Dataset & $\rho$ & $\Delta\rho$ & $z(\rho)$ & Pctile & JSD & Boot JSD \\
\midrule
MUTAG
  & 0.976 & \textcolor{red}{$-$0.014} & \textcolor{red}{$-$2.54}
  & 0\% & 0.035 & 0.021 \\
PTC\_MR
  & 0.950 & \textcolor{red}{$-$0.023} & \textcolor{red}{$-$1.49}
  & 10\% & 0.063 & 0.022 \\
ENZYMES
  & 0.987 & \textcolor{red}{$-$0.005} & \textcolor{red}{$-$2.78}
  & 4\% & 0.014 & 0.010 \\
PROTEINS
  & 0.988 & $+$0.005 & $+$0.62
  & 60\% & 0.022 & 0.016 \\
NCI1
  & 0.990 & $+$0.000 & $+$0.43
  & 50\% & 0.010 & 0.012 \\
\bottomrule
\end{tabular}
\end{table}

\noindent\textbf{Key finding: High $\rho$ is within the bootstrap range.}
For MUTAG, PTC\_MR, and ENZYMES, the model's Spearman $\rho$ falls
\emph{below} the bootstrap median ($\Delta\rho < 0$), placing the
model below the bootstrap distribution's center on these three
datasets (negative $z(\rho)$ throughout, with ENZYMES the most
negative under the tight no-cap bootstrap reference even though its
$|\Delta\rho|$ is smallest).
For PROTEINS and NCI1, the model is at or slightly above the bootstrap
median, but within the bootstrap band shown in Figure~\ref{fig:bootstrap-comparison}.
This result implies that the high distributional alignment observed in
Table~\ref{tab:main-results} is largely a statistical consequence of
memorization: if the model reproduces most training graphs, the subgraph
distribution alignment follows trivially.
\emph{High $\rho$ alone cannot be interpreted as evidence of structural
learning beyond memorization.}

\begin{figure}[t]
  \centering
  \includegraphics[width=\linewidth]{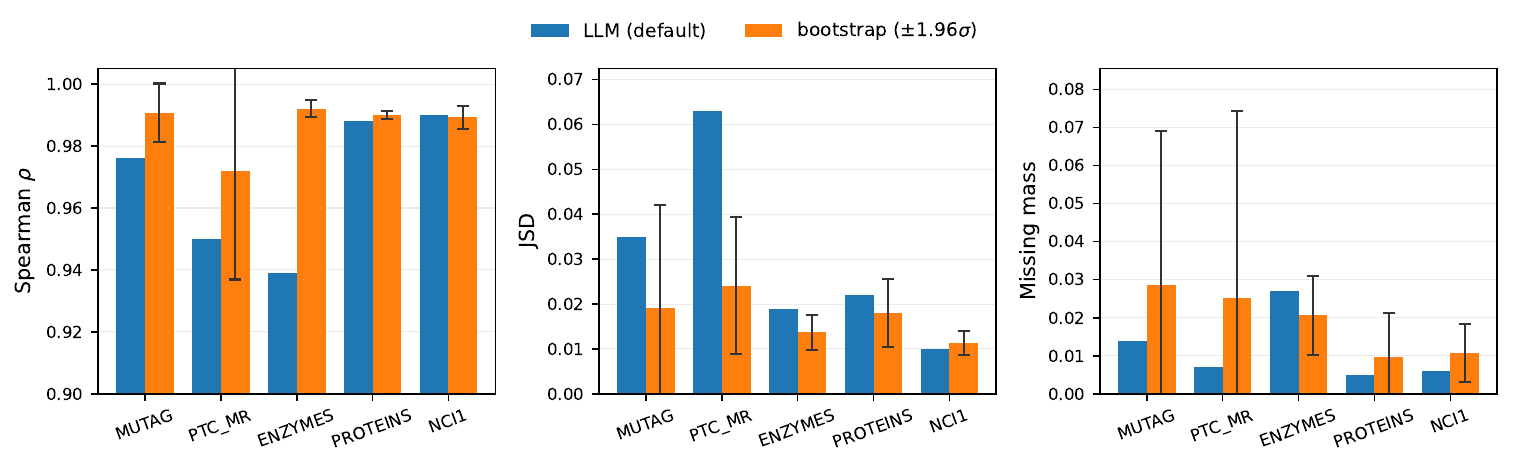}
  \caption{LLM (DFS canonical, \texttt{both\_default}, single training
    seed = 42) vs.\ graph-level bootstrap baseline on the five TU
    benchmarks. Bars on each panel
    show the model and bootstrap mean; error bars on the bootstrap
    bar are $\pm 1.96\,\sigma$ over $10$ bootstrap repeats (a Gaussian
    approximation of a $95\%$ band; the raw $5$--$95$ percentile
    would be too noisy at $n_{\mathrm{boot}} = 10$). Bootstrap matches
    or exceeds the model on Spearman~$\rho$ for MUTAG, PTC\_MR, and
    ENZYMES, and is competitive with the model on JSD and missing
    mass throughout, visualizing the calibrated reading of
    Table~\ref{tab:calibrated}.}
  \label{fig:bootstrap-comparison}
\end{figure}

The DGMG serialization shows a similar pattern
(Appendix~\ref{app:calibrated-dgmg}, Table~\ref{tab:calibrated-dgmg}),
with PROTEINS-DGMG and (under no-cap gBolt; Limitations~(g))
ENZYMES-DGMG nominally exceeding the bootstrap at the $100$th and
$80$th percentile ($z = +1.49$ and $+0.69$).\footnote{The absolute
gaps are small ($\Delta\rho = +0.009$ and $+0.001$), within
decoding-seed variability, so neither is strong evidence of
structural learning beyond memorization;
see Appendix~\ref{app:calibrated-dgmg}.}

\noindent\textbf{Frequency-stratified analysis.}
The aggregate metrics mask a dramatic frequency-dependent structure.
Figure~\ref{fig:strata-bars} presents the per-stratum breakdown; full
numeric values are reported in Appendix~\ref{app:dfs-strata}
(Table~\ref{tab:strata}).
\begin{figure}[t]
  \centering
  \includegraphics[width=1.0\linewidth]{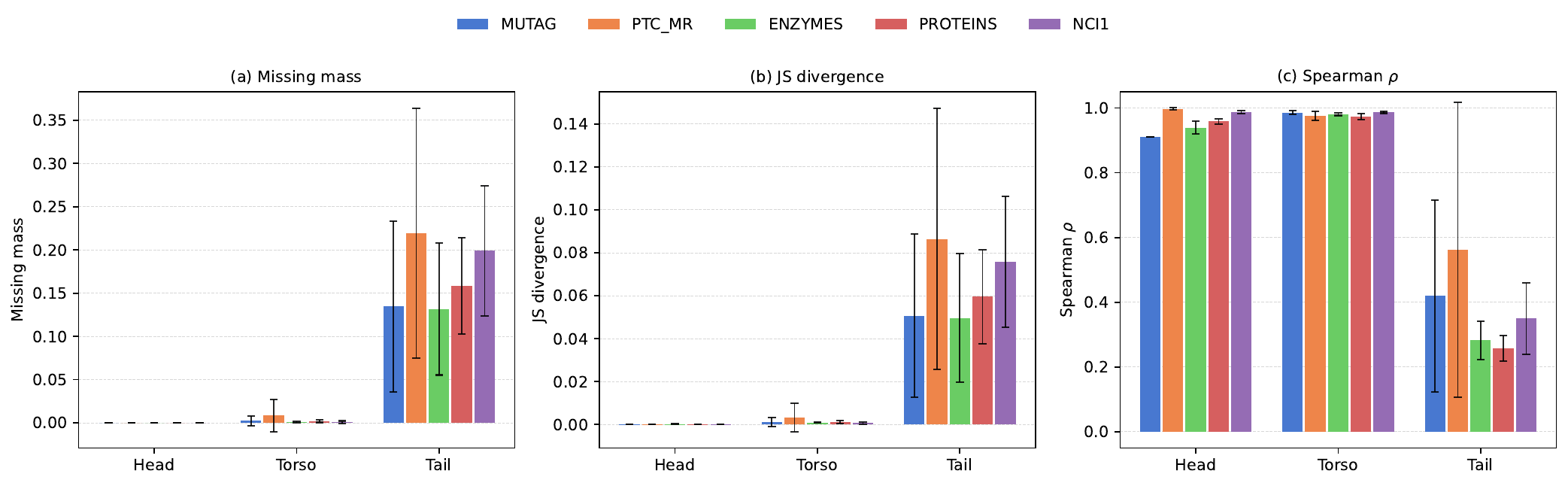}
  \caption{Frequency-stratified metrics (DFS canonical, single
    training seed = 42), bar heights = mean over $n=5$ decoding-seed
    reruns of the same LLaMA-\textsc{small} (132M) checkpoint, error
    bars = sample standard deviation. Missing mass and Spearman~$\rho$ are
    \emph{independent} diagnostics: missing mass measures
    \emph{coverage} (what fraction of stratum mass is omitted), while
    $\rho$ measures \emph{rank} on the patterns the model does emit.
    Head/Torso are reproduced with high fidelity and tight error bars
    across all decoding seeds; Tail patterns exhibit elevated missing
    mass and Jensen--Shannon divergence, and lower Spearman rank
    correlation, with substantially wider error bars---reflecting
    that 1024 samples are too few to densely cover the bottom-decile
    pattern set (per-seed numbers in
    Appendix~\ref{app:decoding-ablation}).}
  \label{fig:strata-bars}
\end{figure}

Head and Torso strata are reproduced with high fidelity (missing mass
$\le 0.8\%$ and $\rho \ge 0.91$ in every condition), but the Tail
stratum shows substantial coverage failures: missing mass climbs to
$13.5$--$22.0\%$, an order of magnitude larger than the Head/Torso gap.
Rank correlation on the patterns the model does produce
(intersection-Spearman Tail $\rho \in [0.26, 0.56]$, mean over $n=5$
decoding-seed reruns) sits well below the $\rho \ge 0.97$ achieved on
Head/Torso. Tail estimates are themselves noisy across decoding
seeds---PTC\_MR and MUTAG Tail $\rho$ swing by $\pm 0.30$--$0.46$
across reruns---reflecting that 1024 samples are too few to densely
cover the bottom-decile pattern set; per-seed values, the more stable
\emph{trainkeys} variant (filling omitted patterns with zero), and
the per-dataset Head--Tail gap $\Delta\rho$ are reported in
Appendices~\ref{app:decoding-ablation}--\ref{app:tail-gap}.
This is a \emph{frequency-dependent generalization gap}: the model
concentrates mass on frequent motifs and systematically
under represents rare ones, with sampling alone unable to recover the
distribution's tail at the available decoding budget.

Two diagnostic checks corroborate the whole-graph memorization reading of
the calibrated baseline comparison above: an NLL hit/miss test
(Mann--Whitney $U$, $p<10^{-5}$ on ENZYMES/NCI1;
Appendix~\ref{app:nll}) and a one-way variance decomposition where
decoding explains $57$--$71\%$ of Spearman variance and training
progress dominates missing-mass variance ($>\!76\%$;
Appendix~\ref{app:factor-decomp}). Repeating the protocol with DGMG
action sequences yields serialization-consistent aggregate patterns
(WL-MMD $<\!0.006$, coverage $>\!0.92$, DGMG memorization $\ge$ DFS
in every dataset; Appendix~\ref{app:serialization}) and reproduces
the qualitative Head--Tail gap (Appendix~\ref{app:dgmg-strata}),
while per-dataset bootstrap percentiles vary between serializations
(Appendix~\ref{app:calibrated-dgmg})---qualitative evidence for
expectation (E4), robustness across the two serializations.

\noindent\textbf{Structural baselines.}
Against three structural baselines, \textbf{DiGress}
~\citep{vignac2023digress} (discrete denoising diffusion),
\textbf{GraphRNN}~\citep{you2018graphrnn}, and
\textbf{DGMG-official}~\citep{li2018dgmg}, DiGress occupies a
low whole-graph memorization but lower subgraph-level alignment regime (recall $\le 2\%$,
$\rho \in [0.62, 0.93]$), confirming that high $\rho$ does not
require memorization. GraphRNN and DGMG-official, being
label-unaware, fail to recover labeled gSpan support on most TU
datasets and leave $\rho$ undefined, even though their unlabeled
degree/orbit MMD match training
(full per-model numbers in Appendix~\ref{app:dgmg-official},
Table~\ref{tab:cross_model}; cross-model MMD aggregate in
Appendix~\ref{app:mmd}).
High $\rho$ alone is therefore insufficient as a verdict, and the
LLM's near-perfect TU alignment still requires the bootstrap
calibration above.

\noindent\textbf{Scaling to PCQM4Mv2.}
Propositions~\ref{prop:tv}--\ref{prop:stat} establish that high $\rho$
is consistent with memorization; the empirical question is whether the
converse---high $\rho$ \emph{without} near-complete verbatim recall of
the corpus---arises at scale, where memorizing $3.10$M unique training
sequences becomes substantially harder at fixed $\sim$132M-parameter
capacity. Applying the same DFS pipeline to PCQM4Mv2, 
exact-match recall \emph{at the whole-graph level} drops to $31.7\%$ 
while \emph{subgraph-level} Spearman remains high 
($\rho = 0.976$, JSD $= 0.044$, uniqueness $= 1.000$); 
the full scale-comparison is in Appendix~\ref{app:pcqm}
(Table~\ref{tab:pcqm-scaling}). The TU bootstrap is uninformative at
this scale; a novel-only subset analysis on a fixed $10{,}000$-graph
subsample tracks the all-generation Spearman within $0.026$ across
the capacity sweep (Appendix~\ref{app:pcqm-novel-only}).
The Tail stratum still degrades sharply (Tail $\rho = 0.41\!\pm\!0.22$,
missing mass $0.07\!\pm\!0.02$; Table~\ref{tab:pcqm-strata}), while
same-corpus DiGress collapses uniformly (Head missing mass already
$0.47$, Head $\rho = 0.54$), so the ``rare-pattern-only'' failure
mode is specific to the LLM regime. Aggregate whole-graph MMDs stay
in the same low band as on TU (Appendix~\ref{app:mmd})---a regime
where Cor~\ref{cor:non-identifiability} makes the bootstrap and
stratification of \S\ref{sec:method} the load-bearing diagnostics,
not aggregate MMD.

\begin{table}[t]
\centering
\caption{PCQM4Mv2 Tail-stratum diagnostics (single training seed = 42;
  mean $\pm$ std over 5 random $10$k training-reference subsamples;
  reference-side variance only). LLM-DFS preserves Head/Torso almost
  perfectly and loses ground only in the Tail; DiGress collapses
  uniformly. Full breakdown and decoding-stage variance discussion
  in Appendix~\ref{app:pcqm} (Table~\ref{tab:pcqm-strata-full}).}
\label{tab:pcqm-strata}
\small
\begin{tabular}{lccc}
\toprule
Model & Tail Missing & Tail JSD & Tail $\rho$ \\
\midrule
LLM-DFS & $0.069 \pm 0.017$ & $0.026 \pm 0.006$ & $0.406 \pm 0.222$ \\
DiGress & $0.860 \pm 0.033$ & $0.489 \pm 0.029$ & $0.046 \pm 0.212$ \\
\bottomrule
\end{tabular}
\end{table}

\begin{figure}[t]
  \centering
  \includegraphics[width=\linewidth]{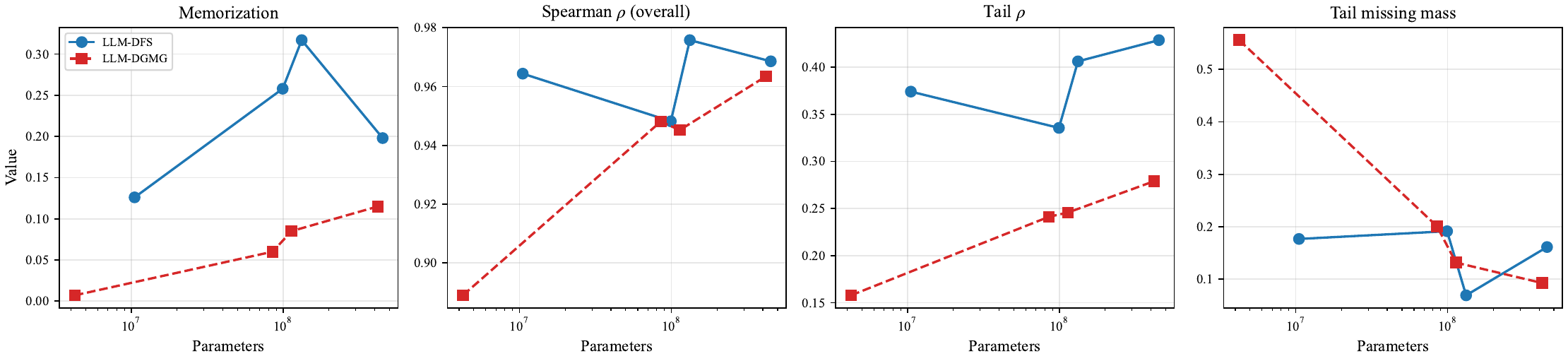}
  \captionsetup{hypcap=false}
  \captionof{figure}{Cross-model PCQM4Mv2 scaling. Memorization, overall
    $\rho$, Tail $\rho$, and Tail missing mass plotted against
    parameter count for LLM-DFS (4 sizes, blue) and LLM-DGMG (4
    sizes, red). Distribution-alignment metrics sit in a narrow band
    across the $43\times$ DFS sweep, while memorization and missing
    mass move with capacity. The DGMG four-point sweep reproduces
    the same trend under a different serialization.}
  \label{fig:pcqm-scaling}
\end{figure}

\noindent\textbf{Capacity sweep and cross-serialization robustness.}
The DFS capacity sweep ($10.5$M--$451$M, $43\times$) on PCQM4Mv2
leaves subgraph-level $\rho \in [0.95, 0.98]$ and whole-graph
WL-MMD $\in [0.0014, 0.0023]$ in narrow bands, with Tail $\rho$
peaking at $\approx 0.43$. A four-point DGMG sweep exhibits the
same capacity-dependent trends under a different serialization
(Figure~\ref{fig:pcqm-scaling}, Appendix~\ref{app:pcqm}), supporting
expectation~(E4)---robustness across the two serializations we test.
\section{Discussion and Conclusion}
\label{sec:discussion}
Our results suggest a conditional answer to our central question---whether graph language models go beyond memorization: yes, but only at sufficient scale and primarily for
frequent substructures.

On small TU benchmarks, high subgraph-level
alignment is largely explained by whole-graph memorization: the model's
Spearman correlation falls within, or close to, the distribution induced
by a non-learning bootstrap baseline. In this regime, aggregate fidelity
is therefore not evidence of structural acquisition.
At PCQM4Mv2 scale, however, the regime shifts. Exact-match recall drops sharply,
yet the generated graphs preserve the rank structure of frequent
subgraph statistics, even when evaluation is restricted to
novel-only generations. This decoupling between whole-graph recall and
subgraph-level alignment provides evidence that autoregressive graph
language models can internalize structural regularities rather than
merely replay training graphs. The same qualitative crossover appears
under both canonical DFS codes and DGMG action sequences, suggesting
it is not a serialization artifact.
By contrast, non-LLM baselines occupy different territory: 
the label-aware DiGress baseline achieves no exact-match recall yet also fails to recover frequent-subgraph distributions, while label-unaware baselines (GraphRNN, DGMG-official) cannot share the labeled support that defines alignment (\S\ref{sec:results}, structural baselines). High subgraph-level alignment with low whole-graph recall therefore appears specific to large-scale graph language models in our scope.

This structural acquisition is uneven: Head and Torso patterns are
reproduced reliably while Tail patterns remain poorly covered across
datasets and serializations, so \emph{graph language models behave
as implicit neural graph miners only in a partial sense---recovering
the dominant support structure but not the rare-substructure tail}.
Memorization-aware diagnosis thus requires a non-learning reference
and frequency-stratified evaluation; Tail coverage is the natural
next step.

\section*{Acknowledgments}
This work was supported by JSPS, KAKENHI Grant Number JP25H01112, Japan
and JST, CREST Grant Number JPMJCR22D3, Japan.
\bibliographystyle{plainnat}
\bibliography{references}

@article{shervashidze2011weisfeiler,
  author  = {Nino Shervashidze and Pascal Schweitzer and Erik Jan van Leeuwen and Kurt Mehlhorn and Karsten M. Borgwardt},
  title   = {Weisfeiler-Lehman Graph Kernels},
  journal = {Journal of Machine Learning Research},
  year    = {2011},
  volume  = {12},
  number  = {77},
  pages   = {2539--2561},
  url     = {http://jmlr.org/papers/v12/shervashidze11a.html}
}

@inproceedings{yan2002gspan,
  author    = {Xifeng Yan and Jiawei Han},
  title     = {gSpan: Graph-Based Substructure Pattern Mining},
  booktitle = {Proceedings of the IEEE International Conference on Data Mining (ICDM)},
  publisher = {IEEE},
  year      = {2002},
  pages     = {721--724},
  doi       = {10.1109/ICDM.2002.1184038}
}

@misc{zhou2017gbolt,
  author       = {Keren Zhou},
  title        = {{gBolt}: a {C}++ implementation of the gSpan algorithm},
  year         = {2017},
  publisher    = {GitHub},
  howpublished = {\url{https://github.com/Jokeren/gBolt}},
  note         = {BSD 2-Clause License}
}

@inproceedings{you2018graphrnn,
  title     = {GraphRNN: Generating Realistic Graphs with Deep Auto-regressive Models},
  author    = {You, Jiaxuan and others},
  booktitle = {Proceedings of the 35th International Conference on Machine Learning},
  year      = {2018}
}

@inproceedings{goyal2020graphgen,
  title     = {GraphGen: A Scalable Approach to Domain-agnostic Labeled Graph Generation},
  author    = {Goyal, Nikhil and Jain, Harsh Vardhan and Ranu, Sayan},
  booktitle = {Proceedings of The Web Conference 2020},
  pages     = {1253--1263},
  year      = {2020}
}

@article{li2018dgmg,
  title   = {Learning Deep Generative Models of Graphs},
  author  = {Li, Yujia and Vinyals, Oriol and Dyer, Chris and Pascanu, Razvan and Battaglia, Peter W.},
  journal = {arXiv preprint arXiv:1803.03324},
  year    = {2018}
}

@inproceedings{vignac2023digress,
  title     = {DiGress: Discrete Denoising Diffusion for Graph Generation},
  author    = {Vignac, Clement and Krawczuk, Igor and Siraudin, Antoine and Wang, Bohan and Cevher, Volkan and Frossard, Pascal},
  booktitle = {International Conference on Learning Representations},
  year      = {2023}
}

@inproceedings{jo2022gdss,
  title     = {Score-based Generative Modeling of Graphs via the System of Stochastic Differential Equations},
  author    = {Jo, Jaehyeong and Lee, Seul and Hwang, Sung Ju},
  booktitle = {International Conference on Machine Learning},
  year      = {2022}
}

@inproceedings{niu2020edpgnn,
  title     = {Permutation Invariant Graph Generation via Score-Based Generative Modeling},
  author    = {Niu, Chenhao and Song, Yang and Song, Jiaming and Zhao, Shengjia and Grover, Aditya and Ermon, Stefano},
  booktitle = {Proceedings of the 23rd International Conference on Artificial Intelligence and Statistics},
  year      = {2020}
}

@inproceedings{simonovsky2018graphvae,
  title     = {{GraphVAE}: Towards Generation of Small Graphs Using Variational Autoencoders},
  author    = {Simonovsky, Martin and Komodakis, Nikos},
  booktitle = {International Conference on Artificial Neural Networks},
  year      = {2018}
}

@inproceedings{you2018gcpn,
  title     = {Graph Convolutional Policy Network for Goal-Directed Molecular Graph Generation},
  author    = {You, Jiaxuan and Liu, Bowen and Ying, Rex and Pande, Vijay and Leskovec, Jure},
  booktitle = {Advances in Neural Information Processing Systems},
  year      = {2018}
}

@inproceedings{shi2020graphaf,
  title     = {{GraphAF}: a Flow-based Autoregressive Model for Molecular Graph Generation},
  author    = {Shi, Chence and Xu, Minkai and Zhu, Zhaocheng and Zhang, Weinan and Zhang, Ming and Tang, Jian},
  booktitle = {International Conference on Learning Representations},
  year      = {2020}
}

@article{zhao2024graphgpt,
  title   = {GraphGPT: Generative Pre-trained Graph Eulerian Transformer},
  author  = {Qifang Zhao and Weidong Ren and Tianyu Li and Hong Liu and Xingsheng He and Xiaoxiao Xu},
  journal = {arXiv preprint arXiv:2401.00529},
  year    = {2024},
  url     = {https://arxiv.org/abs/2401.00529}
}

@article{chen2025flatten,
  title   = {Flatten Graphs as Sequences: Transformers are Scalable Graph Generators},
  author  = {Dexiong Chen and Markus Krimmel and Karsten Borgwardt},
  journal = {arXiv preprint arXiv:2502.02216},
  year    = {2025},
  url     = {https://arxiv.org/abs/2502.02216}
}

@inproceedings{morris2020tudataset,
  title     = {{TUDataset}: A collection of benchmark datasets for learning with graphs},
  author    = {Morris, Christopher and Kriege, Nils M. and Bause, Franka and Kersting, Kristian and Mutzel, Petra and Neumann, Marion},
  booktitle = {ICML 2020 Workshop on Graph Representation Learning and Beyond (GRL+)},
  year      = {2020}
}

@inproceedings{hu2021ogblsc,
  title     = {{OGB-LSC}: A Large-Scale Challenge for Machine Learning on Graphs},
  author    = {Hu, Weihua and Fey, Matthias and Ren, Hongyu and Nakata, Maho and Dong, Yuxiao and Leskovec, Jure},
  booktitle = {NeurIPS Datasets and Benchmarks Track},
  year      = {2021}
}

@inproceedings{carlini2021extracting,
  title     = {Extracting Training Data from Large Language Models},
  author    = {Carlini, Nicholas and Tramer, Florian and Wallace, Eric and Jagielski, Matthew and Herbert-Voss, Ariel and Lee, Katherine and Roberts, Adam and Brown, Tom and Song, Dawn and Erlingsson, Ulfar and others},
  booktitle = {30th USENIX Security Symposium},
  year      = {2021}
}

@inproceedings{carlini2023quantifying,
  title     = {Quantifying Memorization Across Neural Language Models},
  author    = {Carlini, Nicholas and Ippolito, Daphne and Jagielski, Matthew and Lee, Katherine and Tramer, Florian and Zhang, Chiyuan},
  booktitle = {International Conference on Learning Representations},
  year      = {2023}
}

@inproceedings{somepalli2023diffusion,
  title     = {Diffusion Art or Digital Forgery? Investigating Data Replication in Diffusion Models},
  author    = {Somepalli, Gowthami and Singla, Vasu and Goldblum, Micah and Geiping, Jonas and Goldstein, Tom},
  booktitle = {IEEE/CVF Conference on Computer Vision and Pattern Recognition},
  year      = {2023}
}

@inproceedings{thompson2022evaluation,
  title     = {On Evaluation Metrics for Graph Generative Models},
  author    = {Thompson, Rylee and Knyazev, Boris and Ghalebi, Elahe
               and Kim, Jungtaek and Taylor, Graham W.},
  booktitle = {International Conference on Learning Representations},
  year      = {2022}
}

@inproceedings{sanford2024transformer,
  title     = {Understanding Transformer Reasoning Capabilities via Graph Algorithms},
  author    = {Sanford, Clayton and Fatemi, Bahare and Hall, Ethan and Tsitsulin, Anton and Kazemi, Mehran and Halcrow, Jonathan and Perozzi, Bryan and Mirrokni, Vahab},
  booktitle = {Advances in Neural Information Processing Systems},
  year      = {2024}
}

@inproceedings{fatemi2024talklike,
  title     = {Talk like a Graph: Encoding Graphs for Large Language Models},
  author    = {Fatemi, Bahare and Halcrow, Jonathan and Perozzi, Bryan},
  booktitle = {International Conference on Learning Representations},
  year      = {2024}
}

@article{bagal2022molgpt,
  title   = {{MolGPT}: Molecular Generation Using a Transformer-Decoder Model},
  author  = {Bagal, Viraj and Aggarwal, Rishal and Vinod, P K and Priyakumar, U Deva},
  journal = {Journal of Chemical Information and Modeling},
  volume  = {62},
  number  = {9},
  pages   = {2064--2076},
  year    = {2022}
}

@article{touvron2023llama,
  title   = {{LLaMA}: Open and Efficient Foundation Language Models},
  author  = {Touvron, Hugo and others},
  journal = {arXiv preprint arXiv:2302.13971},
  year    = {2023}
}

@inproceedings{loshchilov2019decoupled,
  title     = {Decoupled Weight Decay Regularization},
  author    = {Loshchilov, Ilya and Hutter, Frank},
  booktitle = {International Conference on Learning Representations},
  year      = {2019}
}

@inproceedings{rajbhandari2020zero,
  title     = {{ZeRO}: Memory Optimizations Toward Training Trillion Parameter Models},
  author    = {Rajbhandari, Samyam and others},
  booktitle = {Proceedings of the International Conference for High Performance Computing, Networking, Storage and Analysis},
  year      = {2020}
}


\section*{Reproducibility Statement}
\label{sec:reproducibility}

We aim to make every reported number reproducible from the public 
companion repository.

\paragraph{Datasets.}
TU-benchmark graphs (MUTAG, ENZYMES, NCI1, PROTEINS, PTC\_MR) are 
obtained through PyTorch Geometric's \texttt{TUDataset} loader using 
the canonical splits. PCQM4Mv2 is obtained from the OGB v1.3.5 
release; we use the 2D-graph-only split. All preprocessing
(DFS-code conversion, DGMG action sequence construction, and gSpan
pattern extraction with $\sigma=0.1$, executed via a patched
gBolt~\citep{zhou2017gbolt} C++ engine---max-vertices cap,
projection-memory safeguard, and extended DFS-code output for
canonical-form parity with the reference Python gSpan---wrapped by
an anonymized Python interface) is included in the companion
repository.

\paragraph{Hardware.}
All reported experiments run on a single NVIDIA A100 80\,GB (PCIe)
with bf16 mixed precision via DeepSpeed ZeRO-1. Per TU dataset,
training to 80M non-padding tokens takes $\approx$5.5\,GPU-hours at
$\approx$11\,it/s with batch size 16 (DGMG: 4). The multi-seed TU
sweep (5 datasets $\times$ 2 serializations $\times$ 2 seeds $=$ 20
runs, including the \texttt{seed=42} run) amounts to
$\approx$110\,GPU-hours; PCQM4Mv2 (single seed) adds the same
order of magnitude depending on epoch count and model size. Beyond
these headline figures, preliminary architecture and hyperparameter
sweeps over LLaMA sizes, decoding-configuration ablations, and
DeepSpeed/serialization debugging---together with abandoned and
failed runs that were not retained for the final analysis---approximately
doubled the total project compute. We do not report exact wall-clock
for these exploratory runs because logs of failed and discarded
configurations were not systematically preserved; the companion
repository contains only the configurations and seeds used for the
reported numbers.

\paragraph{Seeds.}
Training uses a global seed (\texttt{42}, \texttt{1337}, 
\texttt{2024} for the three TU runs), passed through to 
\texttt{transformers.set\_seed} and to the SFT trainer 
(\texttt{seed} and \texttt{data\_seed}). PCQM4Mv2 uses a single 
seed (\texttt{42}) due to compute budget; its variance is 
documented qualitatively in the limitations. Generation samplers 
expose their own seed through the evaluation pipeline.

\paragraph{Confidence intervals.}
The main-text TU tables report mean $\pm$ sample std over $n=5$ 
decoding seeds at training seed $=42$; the same checkpoint is 
fixed across decoding draws so the variance reflects sampling 
noise alone (per-dataset breakdown in 
Appendix~\ref{app:decoding-ablation}, 
Table~\ref{tab:decoding-ablation}). A full two-way ablation 
across both training seeds $\{42, 1337\}$ and the same five 
decoding seeds ($100$ evaluations) is reported in 
Table~\ref{tab:multiseed-tail-all} and shows that the 
training-seed shift on the decoding-mean is bounded by $0.05$ on 
Tail miss and $0.10$ on Tail $\rho_\cap$, inside the 
decoding-seed std on every cell---i.e.\ decoding-stage noise 
dominates training-seed noise. The seed-2024 run is queued and 
will close the $n{=}3$ Student-$t$ aggregate at the camera-ready 
stage. Bootstrap baselines on TU use $10$ resamples and report 
$\pm 1.96\sigma$ bands; PCQM4Mv2 is single-training-seed by 
design, its larger reference set ($\sim$$10^4$ tail patterns) 
absorbs most sampling noise, and we quote $5$-subsample standard 
errors where applicable. Tables and figures are produced 
automatically from the aggregated artifacts.

\paragraph{Code release.}
Code and training scripts will be released 
under the MIT License at the (currently anonymized) URL listed in 
the supplementary material upon acceptance. The repository includes 
a single-command re-run script that idempotently trains, evaluates, 
and exports LaTeX tables and figures for the full benchmark matrix.

\section*{Broader Impact}
\label{sec:broader-impact}

Autoregressive generators of molecular and protein-interaction graphs share
the dual-use profile of generative models in chemistry more broadly: a
model that learns to extend small fragments into novel valid molecules
could in principle be redirected toward toxic or controlled-substance
design. Two factors mitigate the immediate risk in this work. First, our
released models are trained on public academic graph benchmarks
(TUDatasets and the OGB PCQM4Mv2 split) without property conditioning
for toxicity, potency, or controlled-substance objectives. Second, the
central scientific contribution is a calibrated evaluation framework
for graph language models---a diagnostic for memorization vs
structural alignment---rather than a stronger generator; the
framework is equally useful for auditing generators with safety
properties (e.g., flagging memorization of restricted training
molecules) as for building them.
We nonetheless recommend that downstream practitioners extending these
methods to property-conditioned chemistry benchmarks adopt access controls
and standard safety filters before releasing generators or generated
samples.

\makeatletter
\let\NGM@origsection\section
\renewcommand\section{\FloatBarrier\NGM@origsection}
\makeatother
\appendix

\section{Limitations}
\label{app:limitations}

\textbf{(a) Serialization dependence.}
Memorization detection is string-based and therefore
serialization-dependent. We test two formats (canonical DFS code and
DGMG action sequences) and report the cross-serialization agreement
as robustness evidence rather than a serialization-invariance claim
(Appendix~\ref{app:serialization}); broader serialization sensitivity
(SMILES, random walks, adjacency-based formats) and a
graph-isomorphism-based detector are left for future work.

\textbf{(b) Compute-optimal scaling.}
\textsc{large} (451M for DFS, 419.7M for DGMG) is trained for $12$
epochs under a fixed compute budget while
\textsc{tiny}/\textsc{medium}/\textsc{small} use $50$ epochs (DFS) or
$20$ epochs (DGMG), so tokens-per-parameter is not held constant
($9$ tok/param for DFS-\textsc{large}, $5.8$ for DGMG-\textsc{large}).
The Tail-stratum gains at \textsc{large} are therefore likely a lower
bound under both serializations. A controlled compute-optimal scaling
analysis over model size, data size, and training compute is left for
future work.

\textbf{(c) Multi-seed coverage.}
We separate two distinct sources of stochasticity: \emph{training-seed}
variance (different fine-tuning runs) and \emph{decoding-seed}
variance (different sampling draws from the same trained checkpoint),
and characterize both via a full two-way ablation
(Appendix~\ref{app:multiseed},
Table~\ref{tab:multiseed-tail-all}): for every TU dataset and
serialization, both training seeds $\{42, 1337\}$ are evaluated under
five fixed decoding seeds $\{0,1,2,3,4\}$ ($10$ cells $\times 5 = 50$
draws per condition, $100$ total). Decoding-stage noise dominates
on TU: across all $10$ cells the decoding-mean shift between training
seeds is at most $0.05$ on Tail miss and $0.10$ on Tail $\rho_\cap$,
inside the decoding-seed std band on every cell. This empirically
validates reporting decoding-seed averages throughout the paper and
makes the seed-2024 camera-ready run a confirmation rather than a
load-bearing addition. PCQM4Mv2 is single-training-seed; the larger
reference set ($\sim$$10^4$ tail patterns vs $\sim$$10^2$ on TU)
makes decoding-seed variance substantially smaller, but a full
multi-seed PCQM4Mv2 sweep is deferred to the camera-ready stage.

\textbf{(d) Structural baselines.}
Our reference baselines are graph-level bootstrap resamples, which
preserve full graph structure. Degree-preserving randomization or
configuration-model baselines would quantify how much structural
information is needed to reach similar distributional alignment.

\textbf{(e) gSpan sensitivity.}
We swept \texttt{min\_sup} ($\sigma \in \{0.01, 0.1\}$) on three TU
datasets (Appendix~\ref{app:minsup}) and found consistent conclusions,
but \texttt{upper} (maximum subgraph size) is fixed and remains
unexplored; a systematic sweep over both gSpan thresholds is left for
future work, so Tail-stratum claims should be read as anchored to the
$(\sigma, \texttt{upper})$ values used here.

\textbf{(f) DGMG canonicalization is action-canonical, not
isomorphism-canonical.}
The sorted-label canonicalization used for DGMG action sequences
(\S\ref{sec:method-serial}) is deterministic but does not enumerate
graph automorphisms; same-label nodes admitting non-trivial
automorphisms can therefore yield distinct action strings even for
isomorphic graphs. The DGMG memorization rates in
Table~\ref{tab:dgmg-mem} are computed on canonicalized action
strings rather than isomorphism classes, so they lower-bound the
true isomorphism-based recall. The DFS canonical code, in contrast,
is unique up to graph isomorphism by
construction~\citep{yan2002gspan}, so the DFS-side $\mathrm{EM}$ in
Eq.~\eqref{eq:em} matches isomorphism-based recall.

\textbf{(g) Projection-size memory guard in our patched gBolt.}
Our patched gBolt engine (\S\ref{sec:reproducibility}) applies a
fixed \texttt{MAX\_PROJECTION\_SIZE} of $5\!\times\!10^{5}$
occurrences: when an intermediate pattern's projection exceeds this
threshold, the corresponding search branch is skipped to prevent
memory exhaustion on dense corpora. The guard primarily prunes
extensions of small, high-multiplicity seed patterns; rare-substructure
descendants (the Tail stratum, low-projection by construction) are
largely unaffected by the guard, as we verify empirically below.
Because the same pipeline is applied to training, generated, and
bootstrap corpora, the relative metrics on the shared
support---Spearman $\rho$, JSD, missing mass, and the bootstrap
calibration---are unaffected in expectation, and our central claims
(memorization regime on TU, decoupling on PCQM4Mv2, persistent Tail
deficit, cross-serialization agreement) do not depend on the guard.
The absolute pattern counts $|\mathcal{P}_{\mathrm{tr}}|$ and
$|\mathcal{P}_{\mathrm{gen}}|$ reported in
Table~\ref{tab:main-results} should be read as
implementation-side bounds: on ENZYMES the no-cap value of
$|\mathcal{P}_{\mathrm{tr}}|$ is $\approx 0.09\%$ larger
($+84$ patterns out of $95{,}335$) and we report those values in
Table~\ref{tab:strata}; on PROTEINS the no-cap-to-$10^{7}$ shift is
$\approx 2.7\%$ and we keep the default-cap values because the
ratio metric on the shared support cancels (see below); the four
remaining datasets (MUTAG, PTC\_MR, NCI1, PCQM4Mv2) are unaffected
by the guard.

We further verified this empirically across all six paper datasets.
Rerunning the full evaluation pipeline with the guard raised to
$10^{9}$ (effectively no cap) on MUTAG, PTC\_MR, NCI1, ENZYMES, and
the PCQM4Mv2 subsamples, and to $10^{7}$ on PROTEINS (the no-cap
binary exhausts memory on PROTEINS' denser projections; the $10^{7}$
value is $20\times$ the original $5\!\times\!10^{5}$ cap and suffices
to expose the asymmetric vs symmetric distinction we report),
we observe:

\begin{center}
\begin{tabular}{lcr}
\toprule
Dataset & Patterns affected by guard & Head $\rho$ shift \\
\midrule
MUTAG, PTC\_MR, NCI1                    & $0$                     & $0$ \\
PCQM4Mv2 ($10$k subsample, $5\!\times\!$) & $0$                     & $0$ \\
PROTEINS                                & $1{,}955$ ($2.7\%$)     & $-0.001$ \\
ENZYMES                                 & $90$  ($0.09\%$)        & $+0.152$ \\
\bottomrule
\end{tabular}
\end{center}

The guard does not activate on three TU benchmarks or on PCQM4Mv2.
PROTEINS shows a $2.7\%$ pattern shift but the central Head Spearman
remains within noise ($-0.001$): the cap activates symmetrically on
the training and generated corpora, so the ratio metric on the shared
support cancels. ENZYMES is the only dataset where the cap activates
asymmetrically (training side dominated, since the model's generated
sequences are short enough that the cap is not reached on the
generated side), producing a $+0.152$ Head Spearman shift driven by
$667$ training-only Head patterns becoming shared once the guard is
removed. We therefore report the no-cap numbers for ENZYMES throughout
Figure~\ref{fig:strata-bars} and the corresponding ENZYMES rows of
Table~\ref{tab:strata}, while the four other TU datasets and PCQM4Mv2
remain on the default cap. Our central claims (TU memorization regime,
PCQM4Mv2 decoupling, persistent Tail deficit, cross-serialization
agreement) hold under either choice for every dataset. Notably,
lifting the cap on ENZYMES eliminates the Head/Torso missing mass
while leaving the Tail deficit essentially unchanged ($-0.003$ on Tail
missing mass and $\approx 0$ on Tail $\rho$), which isolates the
deficit to rare substructures more cleanly than under the original
cap.

For ENZYMES specifically, the no-cap rerun has been propagated to the
DFS-canonical main results table, the calibrated baseline, the strata
breakdown, the Head--Tail gap, the per-stratum decoding-seed ablation
(DFS rows), and the cross-baseline strata diagnostic
(Tables~\ref{tab:main-results}, \ref{tab:calibrated},
\ref{tab:strata}, \ref{tab:tail-gap}, \ref{tab:decoding-ablation},
and~\ref{tab:cross-baseline-strata}), as well as
Table~\ref{tab:calibrated-dgmg} (single-seed DGMG calibration). The
six-configuration decoding sweep (Table~\ref{tab:full-sweep}), the
$\sigma=0.01$ sensitivity row (Table~\ref{tab:minsup}), and the DGMG
multi-seed strata and two-way ablation
(Tables~\ref{tab:strata-dgmg}, the DGMG rows of
Table~\ref{tab:decoding-ablation}, and the $s_t=1337$ DFS column and
DGMG rows of Table~\ref{tab:multiseed-tail-all}) retain the default
cap because the corresponding no-cap reruns either exceed the
projection-memory budget ($\sigma=0.01$) or were deferred for the
camera-ready release (DGMG multi-seed). The qualitative
Head/Torso/Tail ordering and the Tail-stratum deficit hold under
either choice on every dataset, as the symmetry argument and the
empirical TU $+$ PCQM4Mv2 audit above establish.

\section{Proofs of Propositions \ref{prop:tv}--\ref{prop:stat} and Corollary \ref{cor:non-identifiability}}
\label{app:proofs}
\begin{proof}[Proof of Proposition \ref{prop:tv}]
For any measurable $A \subseteq \bar{\mathcal{G}}$,
$\psi_\# Q_X(A) = Q_X(\psi^{-1}(A))$ and similarly for $\psi_\# P_X$.
By the dual characterization
$d_{\mathrm{TV}}(\mu, \nu) = \sup_{B} |\mu(B) - \nu(B)|$, with the
supremum over measurable sets,
\[
  \begin{aligned}
  d_{\mathrm{TV}}(\psi_\# Q_X, \psi_\# P_X)
  &= \sup_{A}
    |Q_X(\psi^{-1}(A)) - P_X(\psi^{-1}(A))| \\
  &\le \sup_{B} |Q_X(B) - P_X(B)| \\
  &= d_{\mathrm{TV}}(Q_X, P_X),
  \end{aligned}
\]
since $\psi$ is measurable and
$\{\psi^{-1}(A) : A \subseteq \bar{\mathcal{G}}\ \text{measurable}\}$ is
a sub-family of the measurable sets in $\mathcal{X}$. This is a special
case of the data-processing inequality for total variation.
\end{proof}

\begin{proof}[Proof of Proposition \ref{prop:stat}]
For any bounded $f:\bar{\mathcal{G}}\to\mathbb{R}$,
\[
  |\mathbb{E}_{Q_G} f - \mathbb{E}_{P_G} f|
  = \Big|\int f \, d(Q_G - P_G)\Big|
  \le \|f\|_\infty\, |Q_G - P_G|(\bar{\mathcal{G}})
  = 2\|f\|_\infty\, d_{\mathrm{TV}}(Q_G, P_G).
\]
Here $|Q_G-P_G|$ denotes the total variation measure of the signed
measure $Q_G-P_G$.
Combining with Proposition \ref{prop:tv} gives the second inequality.
\end{proof}

\begin{proof}[Proof of Corollary \ref{cor:non-identifiability}]
Take $Q_X^{(\mathrm{mem})}$ to be the empirical replay distribution
concentrating on $\{\phi(G_i)\}_{i=1}^n$, so that
$Q_G^{(\mathrm{mem})} = \hat P_G^{\mathrm{boot}}$.
For any $\varepsilon > 0$, let $Q_X^{(\mathrm{gen})}$ be any sequence
model satisfying
$d_{\mathrm{TV}}(Q_X^{(\mathrm{gen})}, P_X) < \varepsilon / (2\|f\|_\infty)$
with positive mass outside $\{\phi(G_i)\}$.
Proposition~\ref{prop:stat} then gives
$|\mathbb{E}_{Q_G^{(\mathrm{mem})}} f - \mathbb{E}_{Q_G^{(\mathrm{gen})}} f|
< \varepsilon$.
\end{proof}

\section{Algorithms}
\label{app:algorithms}
This appendix collects the four algorithms referenced in the main text:
the canonical DFS code construction underlying our serialization map $\phi$;
the bootstrap calibration procedure of \S\ref{sec:method-bootstrap}
that realizes $\hat{P}^{\text{boot}}_G$ empirically; and the SMILES-to-DFS preprocessing pipeline used for PCQM4Mv2.


\begin{algorithm}[H]
\caption{Canonical DFS code via rightmost-path extension
  (referenced from \S\ref{sec:method-serial}).}
\label{alg:canonical-dfs}
\KwIn{labeled graph $G = (V, E, \ell_V, \ell_E)$.}
\KwOut{canonical DFS code $C_{\min}(G)$.}
$\mathcal{C} \gets \emptyset$\;
\ForEach{undirected edge $\{u, v\} \in E$, both directions $(u, v)$ and $(v, u)$}{
  initialize $C \gets [\,\langle 0, 1, \ell_V(u), \ell_V(v), \ell_E(u, v)\rangle\,]$\;
  $R \gets [u, v]$ \tcp*{rightmost path}
  \While{any incident edge of $G$ is not yet in $C$}{
    let $r$ be the last vertex of $R$\;
    \uIf{$r$ has an unused edge to an ancestor on $R$}{
      append the $<_e$-minimum such backward edge to $C$\;
    }\Else{
      scan $R$ from right to left and enumerate forward edges
        $(w, x)$ with $w \in R$ and $x \notin V(C)$\;
      append the candidate with smallest label triple
        $(\ell_V(w), \ell_V(x), \ell_E(w, x))$ and update $R$\;
    }
  }
  $\mathcal{C} \gets \mathcal{C} \cup \{C\}$\;
}
\Return $\arg\min_{C \in \mathcal{C}} C$ under lexicographic order on $<_e$\;
\end{algorithm}


\begin{algorithm}[H]
\caption{Bootstrap calibration of distributional metrics
  (referenced from \S\ref{sec:method-bootstrap}).}
\label{alg:bootstrap-calibration}
\KwIn{training graphs $\mathbf{D}_{\mathrm{tr}}$, target sample size
  $n_{\mathrm{gen}} = |\mathbf{D}_{\mathrm{gen}}|$, number of repeats
  $R$, gSpan parameters $(\sigma, \mathrm{upper})$.}
\KwOut{reference samples
  $\{(\rho^{(r)}, \mathrm{JSD}^{(r)}, \mathrm{MM}^{(r)})\}_{r=1}^{R}$.}
$\mathcal{P}_{\mathrm{tr}}, s_{\mathrm{tr}} \gets
  \mathrm{gSpan}(\mathbf{D}_{\mathrm{tr}}; \sigma, \mathrm{upper})$\;
\For{$r = 1, \ldots, R$}{
  draw $\mathbf{D}^{(r)} \sim \hat P_G^{\mathrm{boot}}$
    with $|\mathbf{D}^{(r)}| = n_{\mathrm{gen}}$
    \tcp*{i.i.d.\ resample from training graphs}
  $\mathcal{P}^{(r)}, s^{(r)} \gets
    \mathrm{gSpan}(\mathbf{D}^{(r)}; \sigma, \mathrm{upper})$\;
  $\rho^{(r)} \gets \mathrm{Spearman}(s_{\mathrm{tr}}, s^{(r)})$ on
    $\mathcal{P}_{\mathrm{tr}} \cap \mathcal{P}^{(r)}$\;
  $\mathrm{JSD}^{(r)} \gets
    \mathrm{JSD}(p_{\mathrm{tr}}, p^{(r)})$,\ \
  $\mathrm{MM}^{(r)} \gets \sum_{g \in
    \mathcal{P}_{\mathrm{tr}} \setminus \mathcal{P}^{(r)}}
    p_{\mathrm{tr}}(g)$\;
}
\Return $\{(\rho^{(r)}, \mathrm{JSD}^{(r)}, \mathrm{MM}^{(r)})\}_{r=1}^{R}$\;
\end{algorithm}


\begin{algorithm}[H]
\caption{PCQM4Mv2 Preprocessing: SMILES to Canonical DFS Code.}
\label{alg:smiles2dfs}
\KwIn{PCQM4Mv2 dataset $\mathcal{D} = \{(s_i, y_i)\}_{i=1}^{N}$, where $s_i$ is a SMILES string and $y_i \in \mathbb{R}$ is the HOMO--LUMO gap.}
\KwOut{Dataset $\mathcal{D}' = \{(c_i, y_i)\}$ of canonical DFS codes with labels.}

\BlankLine
\tcp{Phase 1: SMILES to molecular graph}
\ForEach{$(s_i, y_i) \in \mathcal{D}$ \textnormal{(in parallel)}}{
    $M_i \leftarrow \textsc{MolFromSmiles}(s_i)$ \tcp*{RDKit SMILES parser}
    \If{$M_i = \emptyset$}{
        discard $(s_i, y_i)$\;
        \textbf{continue}\;
    }
    $V_i \leftarrow \{(v, \ell_v) \mid v \in M_i,\ \ell_v = \textsc{GetSymbol}(v)\}$\;
    $E_i \leftarrow \{(u, v, \ell_e) \mid (u,v) \in M_i,\ \ell_e = \textsc{GetBondType}(u,v)\}$\;
    $G_i \leftarrow (V_i, E_i)$ \tcp*{Labeled molecular graph}
}
Remove all entries where $M_i = \emptyset$ \tcp*{Drop NA}

\BlankLine
\tcp{Phase 2: canonical DFS code generation}
\ForEach{$G_i = (V_i, E_i)$}{
    $c^* \leftarrow \infty$\;
    \ForEach{$v_0 \in V_i$}{
        $c \leftarrow \textsc{DFSTraversal}(G_i, v_0)$\;
        \If{$c < c^*$ \textnormal{(lexicographic order)}}{
            $c^* \leftarrow c$\;
        }
    }
    $c_i \leftarrow c^*$ \tcp*{Canonical DFS code}
}

\BlankLine
\KwResult{$\mathcal{D}' = \{(c_i, y_i)\}$, split into train / valid / test.}
\end{algorithm}


\begin{algorithm}[H]
\caption{\textsc{DFSTraversal}$(G, v_0)$.}
\label{alg:dfs-traversal}
\KwIn{Labeled graph $G=(V,E)$, start node $v_0$.}
\KwOut{DFS code sequence $\mathcal{C}$.}

$\mathcal{C} \leftarrow []$,\quad $\textit{visited} \leftarrow \emptyset$,\quad $\textit{order} \leftarrow \{\}$\;

\SetKwFunction{FDFSVisit}{DFSVisit}
\SetKwProg{Fn}{Function}{:}{}
\Fn{\FDFSVisit{$v$, $\textit{parent}$}}{
    $\textit{visited} \leftarrow \textit{visited} \cup \{v\}$\;
    $\textit{order}[v] \leftarrow |\textit{order}|$\;
    \ForEach{$u \in \mathcal{N}(v)$}{
        \eIf{$u \notin \textit{visited}$}{
            \tcp{Forward edge}
            $\mathcal{C}.\text{append}\bigl(\langle \textit{order}[v],\ |\textit{order}|,\ \ell_v,\ \ell_u,\ \ell_{vu},\ \texttt{fwd} \rangle\bigr)$\;
            \FDFSVisit{$u$, $v$}\;
        }{
            \If{$u \neq \textit{parent}$}{
                \tcp{Back edge}
                $\mathcal{C}.\text{append}\bigl(\langle \textit{order}[v],\ \textit{order}[u],\ \ell_v,\ \ell_u,\ \ell_{vu},\ \texttt{bck} \rangle\bigr)$\;
            }
        }
    }
}

\FDFSVisit{$v_0$, $\emptyset$}\;
\Return $\mathcal{C}$\;
\end{algorithm}

\section{Calibrated metrics for DGMG serialization}
\label{app:calibrated-dgmg}

Table~\ref{tab:calibrated-dgmg} reports the bootstrap-calibrated
metrics for the DGMG action-sequence serialization, the DGMG
counterpart of Table~\ref{tab:calibrated} in the main text. The same
qualitative pattern holds: in $3$ of $5$ datasets the model is at or
below the bootstrap median ($\Delta\rho \le 0$ or $z \le 0$). The
exceptions are PROTEINS--DGMG and ENZYMES--DGMG, which sit at the
$100$th and $80$th percentile respectively ($z = +1.49$ and
$z = +0.69$). For PROTEINS the bootstrap is the narrowest of the five
TU datasets (the corpus is large and bootstrap-resampled $\rho$
already saturates near $0.98$), so a $z = +1.49$ deviation translates
to $\Delta\rho = +0.009$ in absolute terms---inside the decoding-seed
variability documented in Appendix~\ref{app:decoding-ablation}---rather
than evidence of meaningfully super-bootstrap structural learning. The
ENZYMES--DGMG entry uses the no-cap gBolt rerun (Limitations~(g));
under the default cap it sits far below the bootstrap
($\Delta\rho = -0.024$, $z = -4.30$, $0$th percentile), so the
above-median appearance here is itself a consequence of removing the
asymmetric pruning, not of stronger structural learning.

\begin{table}[h]
\centering
\caption{Calibrated metrics for DGMG serialization
  (counterpart of Table~\ref{tab:calibrated}; single training seed = 42;
  ENZYMES values are taken from the no-cap gBolt rerun
  to match the no-cap audit in Limitations~(g)).}
\label{tab:calibrated-dgmg}
\begin{tabular}{lcccc}
\toprule
Dataset & $\rho$ & $\Delta\rho$ & $z(\rho)$ & Pctile \\
\midrule
MUTAG    & 0.989 & $-$0.001 & $-$0.30 & 30\% \\
PTC\_MR  & 0.986 & $+$0.008 & $+$0.89 & 70\% \\
ENZYMES  & 0.993 & $+$0.001 & $+$0.69 & 80\% \\
PROTEINS & 0.989 & $+$0.009 & $+$1.49 & 100\% \\
NCI1     & 0.986 & $-$0.004 & $-$1.57 & 10\% \\
\bottomrule
\end{tabular}
\end{table}

\section{NLL hit/miss analysis (full)}
\label{app:nll}

For each unique training sequence we compute the average per-token
negative log-likelihood (NLL) under the trained model, classify
sequences as \emph{hits} (reproduced in the generated set) or
\emph{misses} (not reproduced), and test whether the NLL distributions
differ via the Mann--Whitney $U$ test.

\begin{table}[h]
\centering
\caption{NLL analysis (DFS canonical, single training seed = 42). Hit: training sequences reproduced
  in generation. Miss: not reproduced. MW: Mann--Whitney $U$ test
  $p$-value. *: $p < 0.05$; **: $p < 0.01$; ***: $p < 0.001$.}
\label{tab:nll-app}
\begin{tabular}{lrrccc}
\toprule
Dataset & Hit & Miss & NLL$_{\text{hit}}$ & NLL$_{\text{miss}}$ & MW $p$ \\
\midrule
MUTAG    &  169 &   19 & 2.81 & 2.84 & 0.040* \\
PTC\_MR  &  287 &   51 & 5.40 & 4.98 & $7.3\!\times\!10^{-4}$*** \\
ENZYMES  &  429 &  162 & 3.87 & 4.21 & $2.7\!\times\!10^{-6}$*** \\
PROTEINS &  539 &  533 & 3.91 & 3.74 & 0.060 \\
NCI1     &  883 & 3064 & 5.02 & 5.31 & $2.6\!\times\!10^{-5}$*** \\
\bottomrule
\end{tabular}
\end{table}

For ENZYMES and NCI1, hits have significantly lower NLL than misses
($p < 10^{-5}$), confirming that the model assigns higher probability
to sequences it memorizes. PTC\_MR shows an inverted pattern: misses
have lower NLL than hits ($p < 0.001$), driven by 98.3\% precision
where the few misses are short, high-probability sequences not
selected during sampling. PROTEINS shows no significant difference
($p = 0.06$) with roughly equal hit and miss counts, suggesting a
more uniform probability landscape. The MUTAG entry is reported here
at training seed $42$; the MUTAG multi-seed aggregate
(Appendix~\ref{app:multiseed}, Table~\ref{tab:mutag-multiseed-nll})
gives Mann--Whitney $p = 0.0753 \pm 0.0113$ across two completed
training seeds, consistent with the single-seed reading and with a
memorization-dominated regime in which hit and miss training
sequences carry essentially the same NLL.

\section{Decoding factor decomposition (full)}
\label{app:factor-decomp}

We decompose metric variance into checkpoint (learning) and decoding
(inference) effects via a one-way variance decomposition,
$\eta^2_{\text{ckpt}} = \mathrm{Var}_{\text{ckpt}} /
(\mathrm{Var}_{\text{ckpt}} + \mathrm{Var}_{\text{decode}})$.

\begin{table}[h]
\centering
\caption{Factor effect ratios ($\eta^2$; single training seed = 42). Values indicate the fraction
  of metric variance attributable to each factor.}
\label{tab:factor-app}
\begin{tabular}{lcccccc}
\toprule
& \multicolumn{2}{c}{Spearman $\rho$} &
  \multicolumn{2}{c}{Missing Mass} &
  \multicolumn{2}{c}{Novel Mass} \\
\cmidrule(lr){2-3} \cmidrule(lr){4-5} \cmidrule(lr){6-7}
Dataset & Ckpt & Decode & Ckpt & Decode & Ckpt & Decode \\
\midrule
MUTAG    & 0.87 & 0.13 & 0.99 & 0.01 & 0.60 & 0.40 \\
PTC\_MR  & 0.29 & 0.71 & 0.50 & 0.50 & 0.27 & 0.73 \\
ENZYMES  & 0.37 & 0.63 & 0.89 & 0.11 & 0.01 & 0.99 \\
NCI1     & 0.43 & 0.57 & 0.55 & 0.45 & 0.45 & 0.55 \\
PROTEINS & 0.75 & 0.25 & 0.76 & 0.24 & 0.13 & 0.87 \\
\bottomrule
\end{tabular}
\end{table}

For PTC\_MR, ENZYMES, and NCI1, decoding configuration explains
$57$--$71\%$ of Spearman variance: temperature and sampling strategy
have a larger effect on distributional alignment than training progress
on these datasets. Missing-mass variance is dominated by training
progress in most datasets ($>\!76\%$), reflecting that longer training
reduces pattern omission regardless of decoding strategy. Novelty is
almost entirely determined by decoding ($73$--$99\%$ for PTC\_MR,
ENZYMES, PROTEINS), confirming that exploration of the generation
space is an inference-time phenomenon.

\paragraph{Implication for calibrated evaluation.}
The combination of these $\eta^2$ values and the per-config sweep
in Table~\ref{tab:full-sweep} carries a methodological consequence
for the calibrated diagnostic. The intra-dataset spread induced by
decoding alone is large enough to flip the verdict that the
bootstrap-calibrated comparison of \S\ref{sec:results} (Calibrated baseline comparison)
returns: e.g.\ on MUTAG the same checkpoint moves from $\rho = 0.989$
under \texttt{topk\_only} to $\rho = 0.639$ under \texttt{tp\_high},
spanning the entire bootstrap percentile range. A single decoding
choice therefore underdetermines the memorization--alignment
position; we report \texttt{both\_default} in the main text for
comparability with prior work and document the full sweep here so
that the diagnostic can be re-evaluated under any other choice.

\section{Aggregate whole-graph statistic MMD (full)}
\label{app:mmd}

Table~\ref{tab:mmd-app} reports the four classical whole-graph statistic MMDs
(degree, clustering, orbit, spectral) under DFS canonical
\texttt{both\_default} decoding for every TU benchmark. All values
sit in $\le 0.6 \times 10^{-3}$, i.e.\ below standard reporting
thresholds in graph-generation papers; the discussion below explains
why this is uninformative in the present regime.

\begin{table}[h]
\centering
\caption{Whole-graph statistic MMD metrics (DFS canonical,
  \texttt{both\_default}, single training seed = 42).
  All values $\times 10^3$.}
\label{tab:mmd-app}
\begin{tabular}{lcccc}
\toprule
Dataset & Degree & Clustering & Orbit & Spectral \\
\midrule
MUTAG    & 0.15 & 0.38 & 0.22 & 0.41 \\
PTC\_MR  & 0.04 & 0.01 & 0.16 & 0.36 \\
ENZYMES  & 0.28 & 0.52 & 0.35 & 0.48 \\
PROTEINS & 0.24 & 0.45 & 0.30 & 0.52 \\
NCI1     & 0.08 & 0.12 & 0.10 & 0.25 \\
\bottomrule
\end{tabular}
\end{table}

These low values are consistent with high memorization: if most
generated graphs are exact copies of training graphs, their aggregate
statistics will trivially match training statistics regardless of any
structural learning.

\paragraph{Cross-model MMD on TU.}
Table~\ref{tab:cross-mmd} extends the LLM-DFS numbers above with
the three structural baselines used in
\S\ref{sec:results} (Structural baselines), averaged across the five TU
benchmarks. Both unlabeled baselines (GraphRNN and DGMG-official)
pass the label-agnostic deg/orb checks but fail WL by
$\sim\!10^{2}\!\times$, exposing the labeled-structure failure that
label-free MMD hides.

\begin{table}[h]
\centering
\caption{Cross-model whole-graph statistic MMD on TU (mean over five
  benchmarks, single training seed = 42 for LLM-DFS/DGMG;
  per-dataset values in
  Appendix~\ref{app:dgmg-official}). deg/orb $\times 10^{-3}$, WL
  raw. GraphRNN/DGMG-official pass the label-agnostic deg/orb checks
  yet fail WL by $\sim 10^2 \times$, exposing structural failure that
  label-free MMD hides.}
\label{tab:cross-mmd}
\setlength{\tabcolsep}{6pt}
\begin{tabular}{lccc}
\toprule
Model & deg ($\times 10^{-3}$) & orb ($\times 10^{-3}$) & WL \\
\midrule
GraphRNN      & 0.11 & 0.08  & 0.753 \\
DGMG-official & 2.56 & 2.95  & 0.770 \\
DiGress       & 2.07 & 16.52 & 0.026 \\
LLM-DFS       & 0.16 & 0.23  & $<$0.001 \\
\bottomrule
\end{tabular}
\end{table}

\section{Robustness across serializations: WL-kernel and DGMG memorization}
\label{app:serialization}

\paragraph{Per-dataset graph statistics.}
Table~\ref{tab:datasets} reports the per-graph statistics of the
six benchmarks evaluated in this paper. The TU benchmarks span
$188$--$4{,}110$ graphs with average $|V| \in [14, 39]$, while
PCQM4Mv2 contributes $3.75$M graphs at average $|V| = 14.1$,
giving roughly four orders of magnitude of corpus-size separation
between TU and PCQM4Mv2 at comparable per-graph size.

\begin{table}[h]
\centering
\caption{Graph-level dataset statistics. PCQM4Mv2 numbers are reported
from OGB-LSC~\citep{hu2021ogblsc}; we use the canonical-DFS training
split, which contains $3{,}371{,}958$ sequences ($3{,}104{,}677$
unique).}
\label{tab:datasets}
\setlength{\tabcolsep}{5pt}
\begin{tabular}{lrrrrrr}
\toprule
Dataset  & $|G|$        & Avg $|V|$ & Avg $|E|$ & Max $|V|$ & Avg Deg & Density \\
\midrule
MUTAG    &        188   & 17.9  & 19.8 &   28 & 2.19 & 0.138 \\
PTC\_MR  &        344   & 14.3  & 14.7 &   64 & 1.98 & 0.214 \\
ENZYMES  &        600   & 32.6  & 62.1 &  126 & 3.86 & 0.160 \\
PROTEINS & 1{,}113      & 39.1  & 72.8 &  620 & 3.73 & 0.212 \\
NCI1     & 4{,}110      & 29.9  & 32.3 &  111 & 2.16 & 0.089 \\
PCQM4Mv2 & 3{,}746{,}620 & 14.1  & 14.6 &   53 & 2.07 & 0.147 \\
\bottomrule
\end{tabular}
\end{table}

\paragraph{Cross-serialization robustness.}
We compare canonical DFS code and DGMG action sequences using
format-agnostic WL-kernel metrics (Table~\ref{tab:kernel-app}). WL-MMD
values stay below $0.006$ and coverage above $0.92$ for every dataset
under both serializations, indicating that the two formats produce
graph distributions of comparable distance to the training set in this
metric. Using a whole-graph memorization detector, DGMG memorization is
at least as high as DFS in every dataset (Table~\ref{tab:dgmg-mem});
memorization is therefore not specific to DFS-code serialization.
The frequency-stratified pattern of \S\ref{sec:results} (Frequency-stratified analysis) is
also reproduced under DGMG (Table~\ref{tab:strata-dgmg} below): we
treat this as robustness evidence rather than a
serialization-invariance claim, since the two canonicalizations
remain a small slice of the possible serializations and the contrast
is restricted to what these formats expose to gSpan.

\begin{table}[h]
\centering
\caption{Serialized-sequence statistics. ``Mean / Max~$L$'' is the
sequence length in tokens after subword tokenization, excluding
BOS/EOS. ``$|V_{\mathrm{tok}}|$'' is the vocabulary size of the
serialization-specific tokenizer, excluding the four special tokens
\texttt{<s>}, \texttt{<unk>}, \texttt{<pad>}, \texttt{<mask>}.
PCQM4Mv2 length statistics are estimated on a uniformly random
subsample of $5\!\times\!10^4$ training graphs (vocabulary sizes are
exact). Canonical DFS encodes one event ($v$ / $e$) per token over a
combinatorial node$\times$edge-label vocabulary, while DGMG factorizes
each event into single-symbol actions (\texttt{ADD\_NODE},
\texttt{ADD\_EDGE}, \texttt{TARGET}, \dots) over a tiny vocabulary;
this is why DGMG sequences are several times longer but use far fewer
unique tokens.}
\label{tab:seq-stats-app}
\setlength{\tabcolsep}{6pt}
\begin{tabular}{lrrrrrr}
\toprule
         & \multicolumn{3}{c}{Canonical DFS} & \multicolumn{3}{c}{DGMG} \\
\cmidrule(lr){2-4}\cmidrule(lr){5-7}
Dataset  & Mean $L$ & Max $L$ & $|V_{\mathrm{tok}}|$
         & Mean $L$ & Max $L$ & $|V_{\mathrm{tok}}|$ \\
\midrule
MUTAG    &  19.8 &   33 &     382 &  78.4 &   125 &   39 \\
PTC\_MR  &  14.7 &   71 & 1{,}257 &  61.0 &   273 &   84 \\
ENZYMES  &  60.8 &  149 & 5{,}737 & 192.5 &   545 &  129 \\
PROTEINS &  69.1 & 1{,}049 & 18{,}133 & 226.7 & 3{,}341 &  568 \\
NCI1     &  31.2 &  108 & 6{,}766 & 127.3 &   463 &  148 \\
PCQM4Mv2 &  16.0 &   49 & 12{,}293 &  60.4 &   171 &   83 \\
\bottomrule
\end{tabular}
\end{table}

\begin{table}[h]
\centering
\caption{WL kernel metrics by serialization format (single training seed = 42). WL-MMD: Weisfeiler--Lehman
  maximum mean discrepancy ($\downarrow$). Cov: fraction of training graphs
  ``covered'' by the generated set ($\uparrow$). Div: mean pairwise
  dissimilarity among generated graphs ($\uparrow$).}
\label{tab:kernel-app}
\begin{tabular}{l cc cc cc}
\toprule
& \multicolumn{2}{c}{WL-MMD ($\downarrow$)} &
  \multicolumn{2}{c}{Coverage ($\uparrow$)} &
  \multicolumn{2}{c}{Diversity ($\uparrow$)} \\
\cmidrule(lr){2-3} \cmidrule(lr){4-5} \cmidrule(lr){6-7}
Dataset & DFS & DGMG & DFS & DGMG & DFS & DGMG \\
\midrule
MUTAG    & $-$0.001 & $-$0.001 & 0.995 & 0.989 & 0.720 & 0.712 \\
ENZYMES  &    0.002 &    0.001 & 0.960 & 0.970 & 0.630 & 0.628 \\
NCI1     &    0.000 &    0.000 & 0.926 & 0.942 & 0.698 & 0.700 \\
PROTEINS &    0.001 &    0.006 & 0.973 & 0.946 & 0.544 & 0.570 \\
PTC\_MR  & $-$0.001 & $-$0.002 & 0.980 & 0.986 & 0.452 & 0.427 \\
\bottomrule
\end{tabular}
\end{table}

\begin{table}[h]
\centering
\caption{DGMG memorization rates under whole-graph matching (single training seed = 42). DGMG
  Precision: fraction of generated DGMG sequences whose decoded graph
  is isomorphic to a training graph. DFS Precision is the
  $\mathrm{EM}$ of Eq.~\eqref{eq:em} reproduced from
  Table~\ref{tab:main-results} for comparison.}
\label{tab:dgmg-mem}
\begin{tabular}{lcc}
\toprule
Dataset & DGMG Precision & DFS Precision \\
\midrule
MUTAG    & 100.0\% & 82.8\% \\
PTC\_MR  &  99.4\% & 98.3\% \\
ENZYMES  &  98.9\% & 97.5\% \\
PROTEINS &  91.8\% & 89.1\% \\
NCI1     & 100.0\% & 100.0\% \\
\bottomrule
\end{tabular}
\end{table}

\section{Frequency-Stratified Diagnostics: Full Results}
\label{app:strata}

This appendix consolidates all per-stratum numerical results that the
main-text figure-only frequency-stratified analysis (\S\ref{sec:results}) draws on. We
present them in four subsections that share a common structure: the
DFS-canonical Head/Torso/Tail breakdown
(\S\ref{app:dfs-strata}), its DGMG-serialization counterpart
(\S\ref{app:dgmg-strata}), the Head--Tail gap and Tail-coverage rate
derived from the DFS table (\S\ref{app:tail-gap}), and the
$n=5$ decoding-seed ablation that supplies the $\pm$std values of
both stratified tables (\S\ref{app:decoding-ablation}).

\subsection{DFS canonical}
\label{app:dfs-strata}

Table~\ref{tab:strata} gives the numeric DFS-canonical
Head/Torso/Tail breakdown underlying Figure~\ref{fig:strata-bars}.
The table is here rather than in the main text so that the body text
can focus on the visual frequency-stratified pattern while preserving
exact per-dataset values for comparison and reproducibility.

\begin{table}[h]
\centering
\caption{Frequency-stratified metrics (DFS canonical, single training
  seed = 42), mean $\pm$ std over $n=5$ decoding-seed reruns. Missing:
  fraction of stratum support mass absent in generated set. ``$\rho$''
  is the intersection Spearman; trainkeys-$\rho$ (penalizing omitted
  patterns) is reported alongside in Appendix~\ref{app:decoding-ablation}.}
\label{tab:strata}
\begin{tabular}{ll rrr rrr}
\toprule
& & \multicolumn{3}{c}{Missing Mass} & \multicolumn{3}{c}{Spearman $\rho$} \\
\cmidrule(lr){3-5} \cmidrule(lr){6-8}
Dataset & & Head & Torso & Tail & Head & Torso & Tail \\
\midrule
MUTAG    && 0.000 & 0.003 & 0.135$\pm$0.099 & 0.911 & 0.986 & 0.420$\pm$0.296 \\
PTC\_MR  && 0.000 & 0.008 & 0.220$\pm$0.145 & 0.998 & 0.975 & 0.562$\pm$0.456 \\
ENZYMES  && 0.000 & 0.001 & 0.132$\pm$0.076 & 0.940 & 0.981 & 0.283$\pm$0.059 \\
PROTEINS && 0.000 & 0.002 & 0.159$\pm$0.056 & 0.959 & 0.974 & 0.258$\pm$0.040 \\
NCI1     && 0.000 & 0.001 & 0.199$\pm$0.075 & 0.988 & 0.987 & 0.350$\pm$0.111 \\
\bottomrule
\end{tabular}
\end{table}

\subsection{DGMG action sequences}
\label{app:dgmg-strata}

Table~\ref{tab:strata-dgmg} reproduces the Head/Torso/Tail breakdown
of the main-text frequency-stratified analysis (\S\ref{sec:results}) under DGMG action-sequence
serialization, mirroring the decoding-seed averaging used for
DFS canonical (mean $\pm$ std over $n=5$ reruns of the same
LLaMA-\textsc{small} checkpoint). The Head--Tail asymmetry is
qualitatively the same as under DFS canonical
(Table~\ref{tab:strata}): Head and Torso $\rho$ exceed $0.91$ and
$0.97$ on every dataset, while seed-mean Tail $\rho$ ranges
$0.23$--$0.41$ with Tail missing mass between $15\%$ and $39\%$.
Decoding-stage variance is again concentrated on Tail (Tail $\rho$
std up to $0.46$ on MUTAG/PTC\_MR; Tail miss std up to $0.18$),
consistent with the canonical observation that 1024 samples are
too few to densely cover the bottom-decile pattern set on
small TU benchmarks.

\begin{table}[h]
\centering
\caption{Frequency-stratified metrics (DGMG serialization, single
  training seed = 42), mean $\pm$ std over $n=5$ decoding-seed reruns.
  ``$\rho$'' is the intersection Spearman; trainkeys-$\rho$ in
  Appendix~\ref{app:decoding-ablation}. ENZYMES values retain the
  default cap=500K; the no-cap DGMG multi-seed audit is deferred
  to the camera-ready---see Limitations~(g).}
\label{tab:strata-dgmg}
\begin{tabular}{ll rrr rrr}
\toprule
& & \multicolumn{3}{c}{Missing Mass} & \multicolumn{3}{c}{Spearman $\rho$} \\
\cmidrule(lr){3-5} \cmidrule(lr){6-8}
Dataset & & Head & Torso & Tail & Head & Torso & Tail \\
\midrule
MUTAG    && 0.000 & 0.003 & 0.389$\pm$0.181 & 0.913 & 0.979 & 0.231$\pm$0.462 \\
PTC\_MR  && 0.000 & 0.005 & 0.177$\pm$0.119 & 0.999 & 0.975 & 0.406$\pm$0.328 \\
ENZYMES  && 0.039 & 0.006 & 0.150$\pm$0.105 & 0.981 & 0.990 & 0.379$\pm$0.049 \\
PROTEINS && 0.000 & 0.005 & 0.261$\pm$0.102 & 0.928 & 0.970 & 0.265$\pm$0.028 \\
NCI1     && 0.000 & 0.001 & 0.163$\pm$0.093 & 0.958 & 0.978 & 0.296$\pm$0.031 \\
\bottomrule
\end{tabular}
\end{table}

DGMG shows a similar Head--Tail gap to DFS canonical, providing
robustness evidence across the two serializations studied. Comparing
seed-averaged Tail $\rho$ side-by-side under decoding-seed averaging
($n=5$ each), the per-dataset ordering is mixed and within
overlapping decoding-seed CIs in three of five cases:
DGMG seed-mean is higher on ENZYMES ($0.38$ vs.\ DFS $0.28$) and
PROTEINS ($0.27$ vs.\ $0.26$, essentially equal),
roughly matched on PTC\_MR (DGMG $0.41 \pm 0.33$ vs.\
DFS $0.56 \pm 0.46$, both with very wide CIs driven by the
$\sim$$15$ shared tail patterns), and lower on MUTAG ($0.23$ vs.\
$0.42$) and NCI1 ($0.30$ vs.\ $0.35$). The qualitative Head--Tail
gap is preserved under both formats---supporting expectation (E4)
of cross-serialization robustness---while the per-dataset reordering
appears to reflect serialization-specific tail-coverage biases plus
substantial decoding-stage noise rather than a systematic advantage
of either tokenization.

\subsection{Head--Tail gap and Tail Coverage Rate}
\label{app:tail-gap}

\paragraph{Head--Tail gap.}
Table~\ref{tab:tail-gap} summarizes the Head--Tail gap
$\Delta\rho = \rho_{\text{head}} - \rho_{\text{tail}}$ alongside
the stratum missing-mass values, with $\rho_{\text{tail}}$ averaged
over $n=5$ decoding-seed reruns. The rank-correlation gap is large
on every TU dataset ($\Delta\rho = 0.44$--$0.70$): the model preserves
rank ordering for high-frequency patterns but loses it sharply on
rare patterns, even when only the patterns it does generate are
ranked. The missing-mass deficit is concentrated in the Tail
stratum ($0.132$--$0.220$ vs.\ $\le 0.008$ for Head/Torso),
confirming that the Head--Tail asymmetry is driven by Tail
under-coverage rather than uniformly distributed errors.

\begin{table}[h]
\centering
\caption{Head--Tail gap analysis derived from Table~\ref{tab:strata}
  (single training seed = 42). $\Delta\rho$: Head $\rho$ minus Tail
  $\rho$; larger values indicate greater frequency dependence.}
\label{tab:tail-gap}
\begin{tabular}{lccccc}
\toprule
Dataset & $\rho_{\text{head}}$ & $\rho_{\text{tail}}$
  & $\Delta\rho$ & Miss$_{\text{head}}$ & Miss$_{\text{tail}}$ \\
\midrule
MUTAG    & 0.911 & 0.420 & 0.491 & 0.000 & 0.135 \\
PTC\_MR  & 0.998 & 0.562 & 0.436 & 0.000 & 0.220 \\
ENZYMES  & 0.940 & 0.283 & 0.657 & 0.000 & 0.132 \\
PROTEINS & 0.959 & 0.258 & 0.701 & 0.000 & 0.159 \\
NCI1     & 0.988 & 0.350 & 0.638 & 0.000 & 0.199 \\
\bottomrule
\end{tabular}
\end{table}

\paragraph{Tail Coverage Rate (TCR).}
$\text{TCR} = |\mathcal{P}_{\text{tail}}^{\text{shared}}|
\,/\, |\mathcal{P}_{\text{tail}}^{\text{train}}|$ is the fraction of
tail patterns that appear at least once in the generated set.
Across the TU runs, TCR remains below one in every
dataset--serialization condition, confirming that the Tail deficit
includes complete omission of rare patterns rather than only support
underestimation.

\subsection{Decoding-seed ablation}
\label{app:decoding-ablation}

The TU strata estimates of Table~\ref{tab:strata} are sensitive to
the random sampling at decoding time. For each TU dataset and the
LLaMA-\textsc{small} (132M) checkpoint we re-evaluated the
\texttt{both\_default} configuration ($T=1.0$, top-$p=0.95$,
top-$k=50$, $1024$ samples) under five fixed
\texttt{torch.manual\_seed} values $s \in \{0, 1, 2, 3, 4\}$.
Generation, gSpan mining, and frequency-stratified analysis were
otherwise identical to the main pipeline; bootstrap, kernel, MMD,
and NLL stages were skipped to keep the ablation tractable on CPU.

Table~\ref{tab:decoding-ablation} reports the per-stratum mean and
sample standard deviation across the five seeds. The Head and Torso
strata are highly stable (std $\le 0.014$ for $\rho$, $\le 0.019$ for
missing mass on every dataset). Tail estimates are markedly noisier:
intersection $\rho$ swings by std $0.040$--$0.456$ across reruns of
the same checkpoint, and Tail missing mass by $0.056$--$0.145$. The
trainkeys variant (which fills omitted patterns with zero before
ranking) is consistently more stable than intersection: standard
deviations are typically $1.5$--$5\times$ smaller, and the gap
between the two metrics is largest on datasets with few tail
patterns (PTC\_MR has only $12$--$20$ shared tail patterns per draw,
making intersection-based rank correlation highly contingent on
which exact patterns are generated). We therefore recommend reading
the trainkeys $\rho$ column as the more reliable Tail-rank summary.

\begin{table}[h]
\centering
\caption{Per-stratum decoding-seed ablation ($n=5$ reruns of the same
  LLaMA-\textsc{small} checkpoint with
  \texttt{torch.manual\_seed} $\in \{0,1,2,3,4\}$). Top: DFS
  canonical. Bottom: DGMG action sequences (with
  \texttt{max\_length}=$512$). Values are mean $\pm$ sample std.
  ``$\rho_\cap$'' is intersection Spearman (only patterns shared
  between train and generation are ranked); ``$\rho_{tk}$'' is
  trainkeys Spearman (full train support, with generated$=0$ for
  missing patterns). DFS ENZYMES uses the no-cap gBolt rerun;
  DGMG ENZYMES retains the default cap=500K because the no-cap
  DGMG multi-seed rerun is deferred to the camera-ready---see
  Limitations~(g).}
\label{tab:decoding-ablation}
\setlength{\tabcolsep}{4pt}
\begin{tabular}{l rrr rr}
\toprule
Dataset & Tail miss & Tail $\rho_\cap$ & Tail $\rho_{tk}$
        & Head $\rho_\cap$ & Torso $\rho_\cap$ \\
\midrule
\multicolumn{6}{l}{\emph{DFS canonical}} \\
MUTAG    & 0.135$\pm$0.099 & 0.420$\pm$0.296 & 0.609$\pm$0.108 & 0.911 & 0.986 \\
PTC\_MR  & 0.220$\pm$0.145 & 0.562$\pm$0.456 & 0.560$\pm$0.418 & 0.998 & 0.975 \\
ENZYMES  & 0.132$\pm$0.076 & 0.283$\pm$0.059 & 0.325$\pm$0.049 & 0.940 & 0.981 \\
PROTEINS & 0.159$\pm$0.056 & 0.258$\pm$0.040 & 0.311$\pm$0.038 & 0.959 & 0.974 \\
NCI1     & 0.199$\pm$0.075 & 0.350$\pm$0.111 & 0.380$\pm$0.094 & 0.988 & 0.987 \\
\midrule
\multicolumn{6}{l}{\emph{DGMG action sequence}} \\
MUTAG    & 0.389$\pm$0.181 & 0.231$\pm$0.462 & 0.457$\pm$0.145 & 0.913 & 0.979 \\
PTC\_MR  & 0.177$\pm$0.119 & 0.406$\pm$0.328 & 0.412$\pm$0.374 & 0.999 & 0.975 \\
ENZYMES  & 0.150$\pm$0.105 & 0.379$\pm$0.049 & 0.430$\pm$0.024 & 0.981 & 0.990 \\
PROTEINS & 0.261$\pm$0.102 & 0.265$\pm$0.028 & 0.321$\pm$0.030 & 0.928 & 0.970 \\
NCI1     & 0.163$\pm$0.093 & 0.296$\pm$0.031 & 0.310$\pm$0.044 & 0.958 & 0.978 \\
\bottomrule
\end{tabular}
\end{table}

\paragraph{Implication for single-draw reporting.}
Earlier single-draw values (e.g.\ MUTAG Tail miss $=0.501$,
$\rho_\cap=0.820$ from one decoding lottery in the original
seed-42 eval) sit outside the $n=5$ mean by $>3\sigma$ on Tail miss
and indicate that 1024 samples are not sufficient to densely cover
the bottom-decile pattern set on small TU benchmarks. The
seed-averaged values reported in Table~\ref{tab:strata} attenuate
this by reporting mean$\pm$std, but the residual standard deviations
remain substantial---especially on PTC\_MR, where the tail support
set is so sparse ($\sim$15 patterns) that rank correlation is
intrinsically contingent on the realized draw. We treat the Tail
$\rho$ column on TU as a noisy diagnostic and read it together with
Tail missing mass, which is more stable. The same caveat does not
apply to PCQM4Mv2 (Appendix~\ref{app:pcqm}), where the
$\sim$$10{,}000$-pattern reference set absorbs most of the
sample-level noise and the per-seed mean$\pm$half-width over the
$5$-subsample bootstrap is included in the table.

\paragraph{Cross-training-seed extension.}
The same five-decoding-seed ablation has also been run on the
seed-1337 training checkpoint for every TU dataset and
serialization, giving a full $2 \times 5$ training-seed $\times$
decoding-seed grid per cell ($100$ evaluations total). Per-cell
training-seed shifts on the decoding-mean are reported in
Table~\ref{tab:multiseed-tail-all} and are
$|\text{mean}_{42} - \text{mean}_{1337}| \le 0.05$ on Tail miss and
$\le 0.10$ on Tail $\rho_\cap$, all inside the corresponding
decoding-seed std bands above. We retain the seed-42 row in
Table~\ref{tab:decoding-ablation} as the reference because it
matches the checkpoint that supplies the main-text Tables 4 and 22
and Figure~\ref{fig:strata-bars}; the seed-1337 counterpart is
the right column of Table~\ref{tab:multiseed-tail-all}.

\section{gSpan Minimum Support Sensitivity}
\label{app:minsup}

Table~\ref{tab:minsup} contrasts $\sigma = 0.1$ (the main-text
default) with the more permissive $\sigma = 0.01$ on three TU
datasets (MUTAG, ENZYMES, NCI1). Lowering $\sigma$ exposes more
low-support patterns but leaves the qualitative conclusions
(strong head alignment, degraded tail) intact, as discussed below.

\begin{table}[h]
\centering
\caption{Effect of gSpan minimum support ratio on key metrics
  (DFS canonical, single training seed = 42; ENZYMES $\sigma=0.1$
  values use the no-cap gBolt rerun, while $\sigma=0.01$ retains
  the default cap=500K because the no-cap mining at $\sigma=0.01$
  exceeds the projection-memory budget on ENZYMES---see
  Limitations~(g)).}
\label{tab:minsup}
\begin{tabular}{l cc cc}
\toprule
& \multicolumn{2}{c}{Spearman $\rho$} & \multicolumn{2}{c}{JSD} \\
\cmidrule(lr){2-3} \cmidrule(lr){4-5}
Dataset & $\sigma\!=\!0.1$ & $\sigma\!=\!0.01$ &
  $\sigma\!=\!0.1$ & $\sigma\!=\!0.01$ \\
\midrule
MUTAG   & 0.976 & 0.888 & 0.035 & 0.038 \\
ENZYMES & 0.987 & 0.922 & 0.014 & 0.026 \\
NCI1    & 0.990 & 0.953 & 0.010 & 0.011 \\
\bottomrule
\end{tabular}
\end{table}

Lowering the minimum support from 0.1 to 0.01 reduces Spearman $\rho$
(by 0.02--0.09) and marginally increases JSD.
This is expected: a lower threshold exposes more tail patterns, where
the model performs poorly.
The qualitative conclusion---high head alignment, degraded tail
performance---remains unchanged.

\section{Per-Pattern Support Scatter (All Five TU Datasets)}
\label{app:support-scatter}

Figure~\ref{fig:support-scatter} renders the per-pattern alignment
underlying the aggregate Spearman~$\rho$ for every TU benchmark using
the canonical-DFS LLM final eval. Across all five datasets, shared
patterns concentrate tightly along the $y = x$ diagonal across several
decades of support; train-only and gen-only stripes are thinner for
the larger corpora (NCI1, PROTEINS), where the model has more
opportunities to recover low-support training motifs.

\begin{figure}[h]
  \centering
  \includegraphics[width=\linewidth]{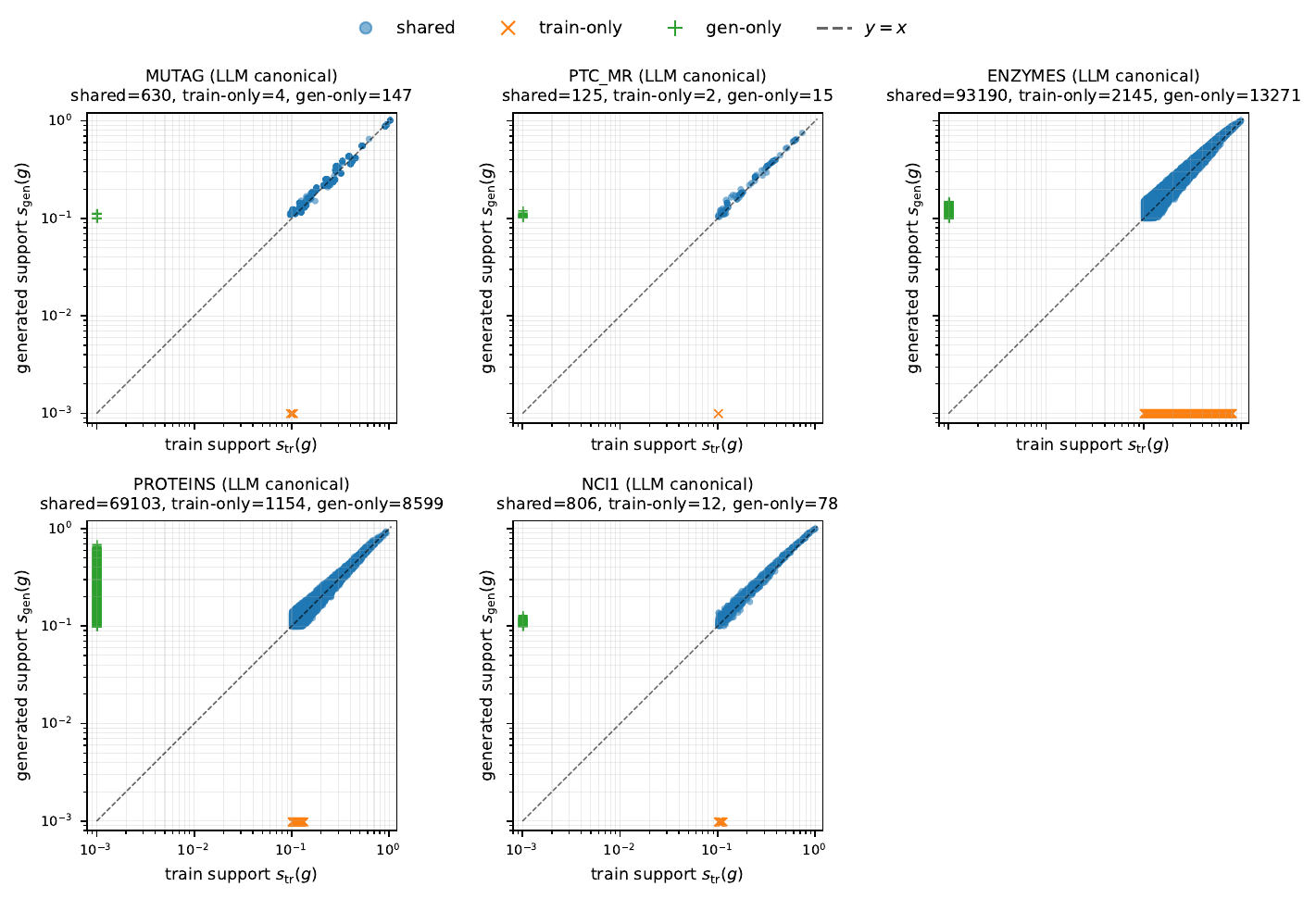}
  \caption{Per-pattern train-vs-generated support on log--log axes
    for all five TU benchmarks. Blue dots: patterns shared by
    $\mathcal{P}_{\mathrm{tr}}$ and $\mathcal{P}_{\mathrm{gen}}$.
    Orange $\times$ (placed at $s_{\mathrm{gen}} = 10^{-3}$):
    train-only patterns. Green $+$ (at $s_{\mathrm{tr}} = 10^{-3}$):
    gen-only patterns. The dashed line is $y = x$. All panels use the
    canonical-DFS LLM final eval (\texttt{both\_default} decoding);
    counts are reported per panel.}
  \label{fig:support-scatter}
\end{figure}

The PCQM4Mv2 capacity sweep yields the same diagnostic at large scale
(Figure~\ref{fig:pcqm-support-scatter}). Each panel renders subsample~0
of the canonical-DFS LLM, with $1024$ generated graphs evaluated
against a fixed $10{,}000$-graph training subsample
($\sigma = 0.1$, $\mathrm{upper} = 8$); supports are normalized by the
respective sample sizes so the diagonal $y = x$ is comparable across
panels. Shared-pattern counts are stable across capacities
($244$--$250$) and the dominant deviation from the diagonal is the
\emph{gen-only} stripe ($71$--$84$ patterns), consistent with the
Tail-stratum analysis of the main-text PCQM4Mv2 scaling (\S\ref{sec:results}): the LLM produces
unseen patterns of comparable mass to the omitted training tail.

\begin{figure}[h]
  \centering
  \includegraphics[width=\linewidth]{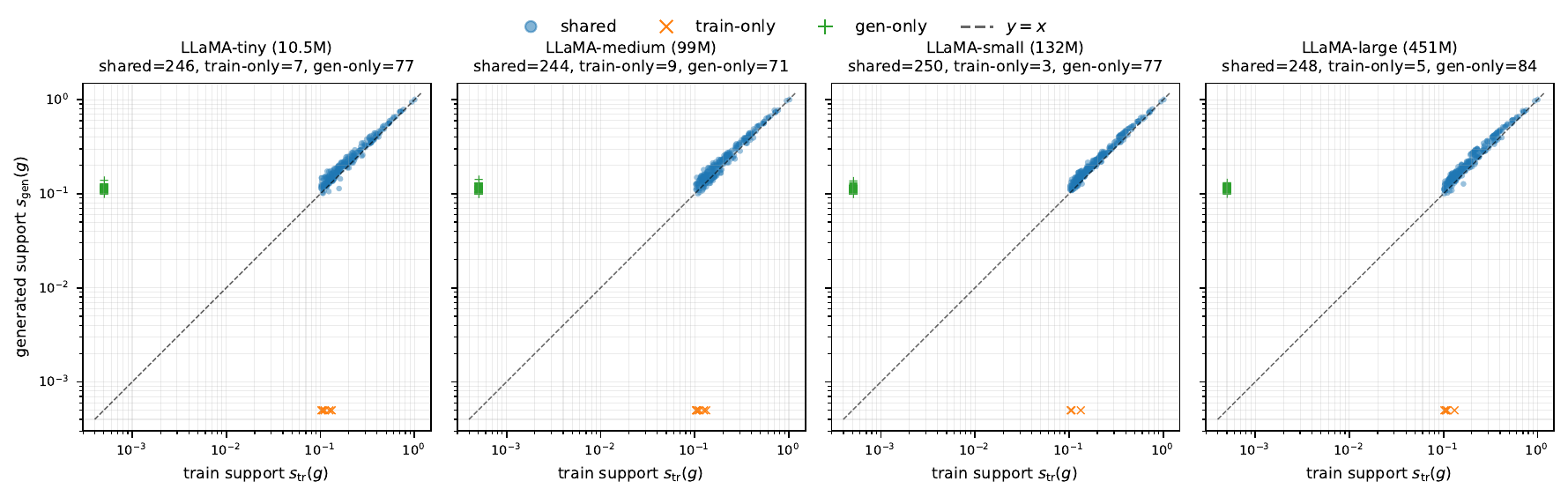}
  \caption{Per-pattern train-vs-generated support on log--log axes for
    PCQM4Mv2 across the four LLaMA capacities (tiny / medium / small /
    large) used in the main-text PCQM4Mv2 scaling (\S\ref{sec:results}). Markers as in
    Figure~\ref{fig:support-scatter}; supports are normalized by sample
    size ($10{,}000$ for train, $1024$ for generation) so the diagonal
    is comparable. Counts of shared / train-only / gen-only patterns
    are reported per panel.}
  \label{fig:pcqm-support-scatter}
\end{figure}

\section{Full Decoding Sweep Results}
\label{app:decode-sweep}

This section reports the per-dataset, per-config decoding sweep
backing the variance-decomposition statement in the main text.
Table~\ref{tab:decode-configs} lists the six decoding configurations
(\texttt{both\_default}, \texttt{nucleus\_only}, \texttt{topk\_only},
\texttt{pure\_sample}, \texttt{tp\_high}, \texttt{tp\_low}) that
sweep temperature, top-$p$, and top-$k$. Table~\ref{tab:full-sweep}
then reports every metric across all five TU datasets and all six
configurations under DFS canonical serialization. The dominant
takeaway is that $\rho$ and missing mass shift substantially with
the decoding choice (e.g.\ MUTAG $\rho$: $0.639$ at \texttt{tp\_high}
vs.\ $0.989$ at \texttt{topk\_only}), so single-config readings
overstate the model's intrinsic capability.

\begin{table}[h]
\centering
\caption{Decoding configurations explored on every TU dataset
(\texttt{both\_default} is the default reported in the main text).}
\label{tab:decode-configs}
\begin{tabular}{llccc}
\toprule
Config & Description & Temp & top-$p$ & top-$k$ \\
\midrule
\texttt{both\_default} & Standard & 1.0 & 0.95 & 50 \\
\texttt{nucleus\_only} & Nucleus sampling & 1.0 & 0.95 & --- \\
\texttt{topk\_only}    & Top-$k$ sampling & 1.0 & --- & 50 \\
\texttt{pure\_sample}  & Unconstrained & 1.0 & --- & --- \\
\texttt{tp\_high}      & High temperature & 1.2 & 0.95 & 50 \\
\texttt{tp\_low}       & Low temperature  & 0.8 & 0.95 & 50 \\
\bottomrule
\end{tabular}
\end{table}

\begin{table}[h]
\centering
\caption{Complete decoding sweep results for all datasets
  (DFS canonical, single training seed = 42; MUTAG multi-seed
  aggregate in Table~\ref{tab:mutag-multiseed-decoding}; ENZYMES rows
  are reported under the default cap=500K for internal consistency
  across the six decoding configurations, while
  Table~\ref{tab:main-results} reports the no-cap value of
  \texttt{both\_default}---see Limitations~(g)).}
\label{tab:full-sweep}
\resizebox{\textwidth}{!}{
\begin{tabular}{llcccccc}
\toprule
Dataset & Config & Unique & Precision & Novelty & $\rho$ & JSD & Missing \\
\midrule
\multirow{6}{*}{MUTAG}
  & both\_default & 0.199 & 0.828 & 0.172 & 0.976 & 0.035 & 0.014 \\
  & nucleus\_only & 0.512 & 0.336 & 0.664 & 0.945 & 0.032 & 0.017 \\
  & topk\_only    & 0.377 & 0.477 & 0.523 & 0.989 & 0.021 & 0.028 \\
  & pure\_sample  & 0.723 & 0.219 & 0.781 & 0.883 & 0.056 & 0.040 \\
  & tp\_high      & 0.937 & 0.107 & 0.893 & 0.639 & 0.428 & 0.057 \\
  & tp\_low       & 0.128 & 0.947 & 0.053 & 0.938 & 0.049 & 0.016 \\
\midrule
\multirow{6}{*}{ENZYMES}
  & both\_default & 0.430 & 0.975 & 0.025 & 0.939 & 0.019 & 0.027 \\
  & nucleus\_only & 0.638 & 0.557 & 0.443 & 0.922 & 0.025 & 0.033 \\
  & topk\_only    & 0.646 & 0.635 & 0.365 & 0.965 & 0.015 & 0.032 \\
  & pure\_sample  & 0.868 & 0.306 & 0.694 & 0.881 & 0.045 & 0.086 \\
  & tp\_high      & 0.977 & 0.055 & 0.945 & 0.452 & 0.345 & 0.266 \\
  & tp\_low       & 0.227 & 1.000 & 0.000 & 0.784 & 0.072 & 0.063 \\
\midrule
\multirow{6}{*}{NCI1}
  & both\_default & 0.862 & 1.000 & 0.000 & 0.990 & 0.010 & 0.006 \\
  & nucleus\_only & 0.849 & 1.000 & 0.000 & 0.984 & 0.013 & 0.007 \\
  & topk\_only    & 0.873 & 0.959 & 0.041 & 0.991 & 0.007 & 0.013 \\
  & pure\_sample  & 0.871 & 0.970 & 0.030 & 0.991 & 0.009 & 0.011 \\
  & tp\_high      & 0.871 & 0.945 & 0.055 & 0.969 & 0.017 & 0.026 \\
  & tp\_low       & 0.690 & 1.000 & 0.000 & 0.915 & 0.040 & 0.025 \\
\midrule
\multirow{6}{*}{PROTEINS}
  & both\_default & 0.591 & 0.891 & 0.109 & 0.988 & 0.022 & 0.005 \\
  & nucleus\_only & 0.594 & 0.875 & 0.125 & 0.986 & 0.022 & 0.008 \\
  & topk\_only    & 0.697 & 0.761 & 0.239 & 0.982 & 0.026 & 0.035 \\
  & pure\_sample  & 0.760 & 0.631 & 0.369 & 0.983 & 0.025 & 0.044 \\
  & tp\_high      & 0.900 & 0.332 & 0.668 & 0.868 & 0.161 & 0.363 \\
  & tp\_low       & 0.300 & 0.850 & 0.150 & 0.789 & 0.084 & 0.058 \\
\midrule
\multirow{6}{*}{PTC\_MR}
  & both\_default & 0.285 & 0.983 & 0.017 & 0.950 & 0.063 & 0.007 \\
  & nucleus\_only & 0.548 & 0.512 & 0.488 & 0.935 & 0.030 & 0.058 \\
  & topk\_only    & 0.438 & 0.701 & 0.299 & 0.947 & 0.032 & 0.043 \\
  & pure\_sample  & 0.713 & 0.358 & 0.643 & 0.925 & 0.068 & 0.133 \\
  & tp\_high      & 0.938 & 0.123 & 0.877 & 0.668 & 0.222 & 0.268 \\
  & tp\_low       & 0.213 & 1.000 & 0.000 & 0.933 & 0.118 & 0.004 \\
\bottomrule
\end{tabular}
}
\end{table}

\section{Checkpoint Trajectories}
\label{app:trajectories}

This appendix combines the MUTAG illustrative table
(\S\ref{app:trajectory}) and the per-dataset trajectory figure
(\S\ref{app:trajectories-all}) that consolidates the same diagnostic
across all five TU benchmarks under both serializations.

\subsection{MUTAG DFS canonical (illustrative)}
\label{app:trajectory}

Table~\ref{tab:trajectory} traces the MUTAG DFS-canonical run across
six training checkpoints (steps $600$--$12000$). Spearman $\rho$
climbs $0.18 \to 0.55 \to 0.89 \to 0.98$ in lockstep with precision
($0.001 \to 0.031 \to 0.180 \to 0.645$), while novelty and the
unique rate decay symmetrically---i.e.\ distributional alignment is
acquired \emph{simultaneously} with memorization rather than after
it, which is the single-dataset evidence behind the
``synchronization'' claim of Appendix~\ref{app:learning-dynamics}.

\begin{table}[h]
\centering
\caption{MUTAG DFS canonical: metrics across training checkpoints
  (single training seed = 42).}
\label{tab:trajectory}
\begin{tabular}{rccccc}
\toprule
Step & Precision & Novelty & Unique & $\rho$ & JSD \\
\midrule
600   & 0.001 & 0.999 & 0.994 & 0.179 & 0.613 \\
1200  & 0.031 & 0.969 & 0.859 & 0.550 & 0.214 \\
2400  & 0.180 & 0.820 & 0.591 & 0.890 & 0.055 \\
6000  & 0.645 & 0.355 & 0.302 & 0.980 & 0.021 \\
9600  & 0.768 & 0.232 & 0.231 & 0.993 & 0.021 \\
12000 & 0.798 & 0.202 & 0.208 & 0.994 & 0.021 \\
\bottomrule
\end{tabular}
\end{table}

\subsection{Per-Dataset Checkpoint Trajectories}
\label{app:trajectories-all}

Figure~\ref{fig:trajectories-all} consolidates the per-checkpoint
Spearman~$\rho$ and Jensen--Shannon divergence trajectories for all
five TU datasets in both DFS canonical and DGMG action serializations
on a common grid. The MUTAG-DFS panel reproduces the trajectory of
Appendix~\ref{app:trajectory} for visual comparison; the remaining
nine combinations are new. Across all ten panels, $\rho$ rises
rapidly and saturates near $0.95$--$0.99$ while JSD decays
correspondingly, mirroring the random~$\to$~rapid
memorization~$\to$~saturation pattern of
Appendix~\ref{app:learning-dynamics}; $\rho$ never rises without a
simultaneous JSD decrease. The raw per-checkpoint values for each
panel (including the unique-rate and precision columns omitted from
the figure) are produced by the same evaluation pipeline as the main
text and are available alongside the LaTeX sources.

\begin{figure}[h]
  \centering
  \includegraphics[width=\linewidth]{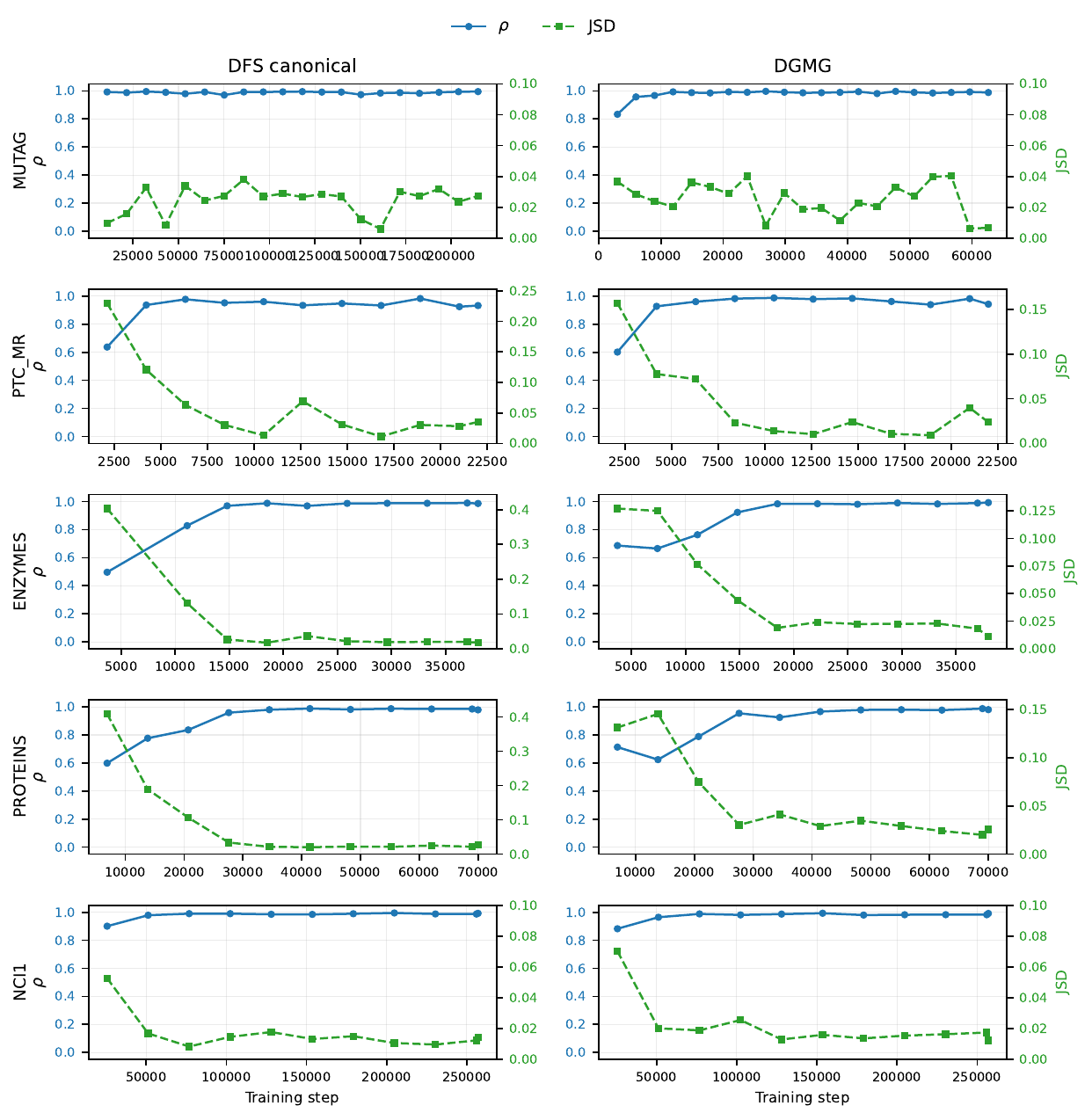}
  \caption{Per-dataset checkpoint trajectories. Rows: TU benchmark
    datasets. Columns: serialization (DFS canonical / DGMG action
    sequence). Solid blue: Spearman~$\rho$ on the gSpan support
    intersection (left axis). Dashed green: Jensen--Shannon
    divergence (right axis, scaled per panel). Precision
    (memorization rate) trajectories are omitted because the
    per-checkpoint detector used in the older DGMG runs pre-dates the
    whole-graph memorization fix described in
    Appendix~\ref{app:serialization}; the corrected end-of-training
    memorization rates are tabulated in
    Appendix~\ref{app:serialization}, Table~\ref{tab:dgmg-mem}.}
  \label{fig:trajectories-all}
\end{figure}

\section{Structural Baselines: DiGress, GraphRNN, DGMG-official}
\label{app:dgmg-official}

\paragraph{Cross-model comparison (full).}
Table~\ref{tab:cross_model} reports per-dataset memorization,
novelty, Spearman~$\rho$, and JSD for the four-model comparison
summarized in \S\ref{sec:results} (Structural baselines). GraphRNN is
unlabeled and yields no support overlap on the larger TU
datasets, so $\rho$ is undefined there.

\begin{table}[t]
\centering
\caption{Cross-model comparison: LLM-DFS, LLM-DGMG (both single
training seed = 42), DiGress, GraphRNN. Memorization is exact-match
graph recall; novelty its complement. GraphRNN is unlabeled and
yields no support overlap on larger TU datasets ($\rho$ undefined).}
\label{tab:cross_model}
\setlength{\tabcolsep}{3pt}
\resizebox{\textwidth}{!}{%
\begin{tabular}{l|cccc|cccc|cccc|cccc}
\toprule
& \multicolumn{4}{c|}{Memorization} & \multicolumn{4}{c|}{Novelty} & \multicolumn{4}{c|}{Spearman \(\rho\)} & \multicolumn{4}{c}{JSD} \\
Dataset & DFS & DGMG & DiGress & GraphRNN & DFS & DGMG & DiGress & GraphRNN & DFS & DGMG & DiGress & GraphRNN & DFS & DGMG & DiGress & GraphRNN \\
\midrule
MUTAG & 0.828 & 1.000 & 0.000 & 0.000 & 0.172 & 0.000 & 1.000 & 1.000 & 0.983 & 0.990 & 0.621 & 0.063 & 0.030 & 0.013 & 0.217 & 0.554 \\
ENZYMES & 0.975 & 0.989 & 0.000 & 0.000 & 0.025 & 0.011 & 1.000 & 1.000 & 0.939 & 0.944 & 0.777 & -- & 0.020 & 0.017 & 0.090 & 0.693 \\
NCI1 & 1.000 & 1.000 & 0.000 & 0.000 & 0.000 & 0.000 & 1.000 & 1.000 & 0.989 & 0.989 & 0.931 & -- & 0.015 & 0.018 & 0.038 & 0.693 \\
PROTEINS & 0.891 & 0.918 & 0.005 & 0.000 & 0.109 & 0.082 & 0.995 & 1.000 & 0.980 & 0.985 & 0.761 & -- & 0.028 & 0.022 & 0.087 & 0.693 \\
PTC\_MR & 0.983 & 0.994 & 0.020 & 0.000 & 0.017 & 0.006 & 0.980 & 1.000 & 0.986 & 0.974 & 0.794 & -- & 0.021 & 0.019 & 0.107 & 0.693 \\
PCQM4Mv2 & 0.317 & 0.085 & 0.000 & -- & 0.683 & 0.915 & 1.000 & -- & 0.976 & 0.945 & 0.247 & -- & 0.044 & 0.061 & 0.401 & -- \\
\bottomrule
\end{tabular}}
\end{table}

\paragraph{Structural-MMD breakdowns.}
The tables below reproduce the cross-baseline comparison together
with structural-MMD breakdowns that are independent of vertex/edge
labels. This view locates GraphRNN and DGMG-official---which do
not emit labeled graphs and therefore have no shared gSpan support
with the labeled training set---on the structural similarity axis.

\begin{table}[t]
\centering
\caption{Appendix: comparison of structural autoregressive baselines on the TU benchmark suite. Memorization, $\rho$ and JSD use the gSpan-based pattern-distribution metrics; structural MMD is computed without label information. DGMG-official is the original implementation by Li et al.\ (2018) at the published scale (\textasciitilde 100K parameters); LLM-DFS/DGMG numbers (single training seed = 42) are provided for reference in the main paper.}
\label{tab:dgmg_official_appendix}
\resizebox{\textwidth}{!}{%
\begin{tabular}{ll|cc|ccc|ccc}
\toprule
 & & \multicolumn{2}{c|}{Memorization-based} & \multicolumn{3}{c|}{gSpan pattern} & \multicolumn{3}{c}{Structural MMD} \\
Dataset & Model & Mem.\ rate & Novelty & Shared & $\rho$ & JSD & WL & deg & orb \\
\midrule
MUTAG & DiGress & 0.000 & 1.000 & 92 & 0.621 & 0.217 & 0.0780 & 0.00475 & 0.07332 \\
MUTAG & GraphRNN & 0.000 & 1.000 & 16 & 0.063 & 0.554 & 0.2373 & 0.00023 & 0.00011 \\
MUTAG & DGMG-official & 0.000 & 1.000 & 16 & 0.069 & 0.549 & 0.2216 & 0.00108 & 0.00163 \\
PTC\_MR & DiGress & 0.020 & 0.980 & 66 & 0.794 & 0.107 & 0.0004 & 0.00225 & 0.00190 \\
PTC\_MR & GraphRNN & 0.000 & 1.000 & 0 & n/a & 0.693 & 1.1748 & 0.00010 & 0.00007 \\
PTC\_MR & DGMG-official & 0.000 & 1.000 & 0 & n/a & 0.693 & 1.1934 & 0.00011 & 0.00067 \\
ENZYMES & DiGress & 0.000 & 1.000 & 229 & 0.777 & 0.090 & 0.0253 & 0.00228 & 0.00286 \\
ENZYMES & GraphRNN & 0.000 & 1.000 & 0 & n/a & 0.693 & 0.5260 & 0.00004 & 0.00003 \\
ENZYMES & DGMG-official & 0.000 & 1.000 & 0 & n/a & 0.693 & 0.5723 & 0.00876 & 0.00943 \\
PROTEINS & DiGress & 0.005 & 0.995 & 265 & 0.761 & 0.087 & 0.0249 & 0.00075 & 0.00149 \\
PROTEINS & GraphRNN & 0.000 & 1.000 & 0 & n/a & 0.693 & 0.4858 & 0.00003 & 0.00004 \\
PROTEINS & DGMG-official & 0.000 & 1.000 & 0 & n/a & 0.693 & 0.5119 & 0.00142 & 0.00125 \\
NCI1 & DiGress & 0.000 & 1.000 & 317 & 0.931 & 0.038 & 0.0021 & 0.00031 & 0.00305 \\
NCI1 & GraphRNN & 0.000 & 1.000 & 0 & n/a & 0.693 & 1.3392 & 0.00015 & 0.00014 \\
NCI1 & DGMG-official & 0.000 & 1.000 & 0 & n/a & 0.693 & 1.3534 & 0.00143 & 0.00175 \\
\bottomrule
\end{tabular}}
\end{table}

\paragraph{Frequency-stratified comparison.}
The Head--Tail asymmetry of the main-text frequency-stratified analysis (\S\ref{sec:results}) is not
LLM-specific: every baseline shows a frequency-dependent failure
mode of its own (Table~\ref{tab:cross-baseline-strata}). DiGress, the
only other label-aware baseline, leaves Head perfectly covered
(Head miss $=0$ in every dataset) but loses $25$--$72\%$ of Tail
mass and yields non-positive Tail $\rho_\cap$ on $4$ of $5$
datasets---larger Tail coverage gaps than LLM-DFS by a factor of
$2$--$5\times$ and lower-or-flipped Tail rank correlation. The
label-unaware baselines (GraphRNN, DGMG-official) emit unlabeled
graphs and therefore share essentially no labeled gSpan support
with the training set on $4$ of $5$ datasets (Head/Torso/Tail miss
$\approx 1$, $\rho_\cap$ undefined when the shared-pattern set is
empty); MUTAG is the lone exception, whose $7$-symbol node
vocabulary admits coincidental label coincidence even from
unlabeled samples. The strata diagnostic therefore (a) confirms
that the LLM-DFS Tail deficit is real, but (b) places it in a
regime that all four baselines fail \emph{worse}---so the diagnostic
discriminates among graph generators rather than reflecting an
LLM-specific bias.

\begin{table}[h]
\centering
\caption{Cross-baseline frequency-stratified diagnostics
  ($\sigma=0.1$, $\mathrm{upper}=8$). LLM-DFS values are mean
  $\pm$ std over $n=5$ decoding seeds
  (Appendix~\ref{app:decoding-ablation}); baselines are single-draw
  ($1024$ generated graphs each). DGMG-official and GraphRNN are
  label-unaware: they emit unlabeled graphs and therefore share
  essentially no labeled gSpan support with the training set
  (Head/Torso/Tail miss $\approx 1$, intersection $\rho$ undefined
  on $4$ of $5$ datasets); we omit those columns rather than
  cluttering the table with ``$1.00$/--'' rows. ``$-$'' indicates
  $\rho_\cap$ is undefined because no patterns are shared in that
  stratum.}
\label{tab:cross-baseline-strata}
\setlength{\tabcolsep}{4pt}
\begin{tabular}{l cc cc cc}
\toprule
        & \multicolumn{2}{c}{Head miss}
        & \multicolumn{2}{c}{Tail miss}
        & \multicolumn{2}{c}{Tail $\rho_\cap$} \\
\cmidrule(lr){2-3}\cmidrule(lr){4-5}\cmidrule(lr){6-7}
Dataset & LLM-DFS & DiGress
        & LLM-DFS & DiGress
        & LLM-DFS & DiGress \\
\midrule
MUTAG    & 0.000 & 0.000 & 0.135$\pm$0.099 & 0.719 & 0.420$\pm$0.296 & $-$0.169 \\
PTC\_MR  & 0.000 & 0.000 & 0.220$\pm$0.145 & 0.251 & 0.562$\pm$0.456 & $-$0.157 \\
ENZYMES  & 0.000 & 0.000 & 0.132$\pm$0.076 & 0.572 & 0.283$\pm$0.059 & $-$0.126 \\
PROTEINS & 0.000 & 0.000 & 0.159$\pm$0.056 & 0.471 & 0.258$\pm$0.040 & $-$0.291 \\
NCI1     & 0.000 & 0.000 & 0.199$\pm$0.075 & 0.613 & 0.350$\pm$0.111 & $+$0.017 \\
\bottomrule
\end{tabular}
\end{table}

\section{Multi-Seed Reproducibility}
\label{app:multiseed}

The reproducibility statement (\S\ref{sec:reproducibility}) targets
mean $\pm$ 95\% Student-$t$ CI over three training seeds
($\{42, 1337, 2024\}$) for every TU result. At submission time, all
five TU datasets have completed $n=2$ training seeds
($\{42, 1337\}$) under both DFS canonical and DGMG action
serializations; seed $2024$ is queued and will be added in the
camera-ready without further structural changes.

\paragraph{Cross-dataset Tail summary (training-seed $\times$ decoding-seed).}
Single-decoding-draw aggregation across two training seeds is
heavily contaminated by decoding-stage noise: e.g.\ on MUTAG
canonical the seed-42 single draw gives Tail miss $0.501$ while the
seed-1337 single draw gives $0.032$, a swing of $0.47$. To
disentangle training-seed and decoding-seed effects we ran a full
two-way ablation: for each (training seed, dataset, serialization)
we re-evaluated the same checkpoint under five fixed
\texttt{torch.manual\_seed} values $\{0,1,2,3,4\}$.
Table~\ref{tab:multiseed-tail-all} reports the resulting decoding-seed
mean $\pm$ sample std for both training seeds. \emph{The training-seed
shift on the decoding-mean is small: across all $10$ TU
dataset/serialization cells, the gap
$|\text{mean}_{42} - \text{mean}_{1337}|$ is at most $0.10$ for
Tail $\rho_\cap$ and at most $0.05$ for Tail miss, and sits inside
the decoding-seed std band on every cell.} Decoding-stage noise
therefore dominates training-seed noise for Tail-stratum estimates
on TU---the empirical reason we report decoding-seed averages
throughout the paper and treat the seed-2024 camera-ready run as a
confirmation rather than a load-bearing addition.

\begin{table}[h]
\centering
\caption{Two-way (training-seed $\times$ decoding-seed) ablation on
  Tail-stratum diagnostics. Each cell is the mean $\pm$ sample std
  over $n=5$ decoding seeds (\texttt{torch.manual\_seed}
  $\in \{0,1,2,3,4\}$) at the same training checkpoint.
  ``$\rho_\cap$'': intersection Spearman, ``$\rho_{tk}$'': trainkeys
  Spearman (Appendix~\ref{app:decoding-ablation}). Across all $10$
  cells the decoding-mean shift between training seeds is
  $\le 0.05$ on Tail miss and $\le 0.10$ on Tail $\rho_\cap$, all
  within the decoding-seed std. ENZYMES $s_t=42$ DFS values use the
  no-cap gBolt rerun (matching Tables~\ref{tab:strata}
  and~\ref{tab:decoding-ablation}); the $s_t=1337$ ENZYMES DFS column
  and ENZYMES DGMG rows retain the default cap because the
  corresponding no-cap multi-seed reruns are deferred to the
  camera-ready---see Limitations~(g).}
\label{tab:multiseed-tail-all}
\setlength{\tabcolsep}{3pt}
\small
\begin{tabular}{ll cc cc cc}
\toprule
& & \multicolumn{2}{c}{Tail miss}
   & \multicolumn{2}{c}{Tail $\rho_\cap$}
   & \multicolumn{2}{c}{Tail $\rho_{tk}$} \\
\cmidrule(lr){3-4}\cmidrule(lr){5-6}\cmidrule(lr){7-8}
Dataset & Mode & $s_t{=}42$ & $s_t{=}1337$
        & $s_t{=}42$ & $s_t{=}1337$
        & $s_t{=}42$ & $s_t{=}1337$ \\
\midrule
MUTAG    & DFS  & 0.135$\pm$0.099 & 0.135$\pm$0.068 & 0.420$\pm$0.296 & 0.402$\pm$0.230 & 0.609$\pm$0.108 & 0.589$\pm$0.093 \\
MUTAG    & DGMG & 0.389$\pm$0.181 & 0.472$\pm$0.204 & 0.231$\pm$0.462 & 0.162$\pm$0.540 & 0.457$\pm$0.145 & 0.399$\pm$0.048 \\
PTC\_MR  & DFS  & 0.220$\pm$0.145 & 0.175$\pm$0.096 & 0.562$\pm$0.456 & 0.642$\pm$0.396 & 0.560$\pm$0.418 & 0.596$\pm$0.341 \\
PTC\_MR  & DGMG & 0.177$\pm$0.119 & 0.196$\pm$0.155 & 0.406$\pm$0.328 & 0.307$\pm$0.436 & 0.412$\pm$0.374 & 0.356$\pm$0.361 \\
ENZYMES  & DFS  & 0.132$\pm$0.076 & 0.141$\pm$0.086 & 0.283$\pm$0.059 & 0.287$\pm$0.069 & 0.325$\pm$0.049 & 0.329$\pm$0.055 \\
ENZYMES  & DGMG & 0.150$\pm$0.105 & 0.170$\pm$0.119 & 0.379$\pm$0.049 & 0.366$\pm$0.056 & 0.430$\pm$0.024 & 0.425$\pm$0.028 \\
PROTEINS & DFS  & 0.159$\pm$0.056 & 0.164$\pm$0.065 & 0.258$\pm$0.040 & 0.263$\pm$0.032 & 0.311$\pm$0.038 & 0.316$\pm$0.037 \\
PROTEINS & DGMG & 0.261$\pm$0.102 & 0.288$\pm$0.149 & 0.265$\pm$0.028 & 0.261$\pm$0.047 & 0.321$\pm$0.030 & 0.324$\pm$0.034 \\
NCI1     & DFS  & 0.199$\pm$0.075 & 0.173$\pm$0.060 & 0.350$\pm$0.111 & 0.348$\pm$0.112 & 0.380$\pm$0.094 & 0.376$\pm$0.095 \\
NCI1     & DGMG & 0.163$\pm$0.093 & 0.184$\pm$0.082 & 0.296$\pm$0.031 & 0.278$\pm$0.092 & 0.310$\pm$0.044 & 0.282$\pm$0.042 \\
\bottomrule
\end{tabular}
\end{table}

\paragraph{Pooled summary (training $+$ decoding, $n{=}10$).}
Pooling decoding draws across both training seeds gives a single
$n{=}10$ estimate per cell: Tail miss $0.135$--$0.430$, Tail
$\rho_\cap$ $0.20$--$0.60$, Tail $\rho_{tk}$ $0.30$--$0.60$, with
sample std typically within $1.05$--$1.10\times$ the
per-training-seed std (i.e.\ very little additional spread beyond
decoding-stage noise alone). This pooled view is the most defensible
single-row Tail summary at submission time; the seed-2024
camera-ready run will extend the training-seed dimension from
$n{=}2$ to $n{=}3$ without changing the dominant noise source.

\paragraph{MUTAG legacy single-decoding aggregate tables.}
The three tables below retain the original $n{=}2$ training-seed
aggregate (single decoding draw per seed, mean $\pm$ 95\%
Student-$t$ CI) for historical comparison with earlier drafts. The
two-way decoding-seed averaging of
Table~\ref{tab:multiseed-tail-all} is the load-bearing report; the
wide CIs in the tables below reflect the small training-seed sample
size ($t_{0.975}(1)=12.7$) and the absence of decoding-seed
averaging.

\paragraph{Decoding sweep.}
Table~\ref{tab:mutag-multiseed-decoding} aggregates the six decoding
configurations of Table~\ref{tab:full-sweep} over the available seeds.
Spearman~$\rho$ stays in $0.94$--$0.98$ across configurations, and
precision saturates near $0.99$ for every config except
\texttt{topk\_only}/\texttt{pure\_sample} ($\approx 0.98$). The
seed-to-seed standard errors on $\rho$ ($\sim$$0.15$) are large
relative to the across-config spread ($\sim$$0.04$), so the
single-config readings in the main text should be interpreted with
this seed-level uncertainty in mind.

\begin{table}[h]
\centering
\caption{MUTAG, DFS canonical, multi-seed aggregate (2 seeds:
  \{42, 1337\}; seed 2024 in progress, see Sec.~\ref{app:limitations};
  mean $\pm$ SE).}
\label{tab:mutag-multiseed-decoding}
\resizebox{\textwidth}{!}{%
\begin{tabular}{lrrrrrr}
\toprule
Config & Unique & Precision & Novelty & $\rho$ & JSD & Missing \\
\midrule
both\_default & 0.1636 $\pm$ 0.0062 & 1.0000 & 0.0000 & 0.9829 $\pm$ 0.0282 & 0.0298 $\pm$ 0.0021 & 0.0193 $\pm$ 0.2157 \\
nucleus\_only & 0.1641 & 1.0000 & 0.0000 & 0.9856 $\pm$ 0.0558 & 0.0171 $\pm$ 0.0139 & 0.0211 $\pm$ 0.2678 \\
topk\_only & 0.1851 $\pm$ 0.0062 & 0.9894 $\pm$ $3.54e-04$ & 0.0106 $\pm$ $3.54e-04$ & 0.9913 $\pm$ 0.0057 & 0.0224 $\pm$ 0.0162 & 0.0424 $\pm$ 0.0507 \\
pure\_sample & 0.1860 $\pm$ 0.0310 & 0.9817 $\pm$ 0.0970 & 0.0183 $\pm$ 0.0970 & 0.9940 $\pm$ 0.0445 & 0.0174 $\pm$ 0.0452 & 0.0378 $\pm$ 0.0187 \\
tp\_high & 0.1816 & 1.0000 & 0.0000 & 0.9792 $\pm$ 0.0290 & 0.0371 $\pm$ 0.0355 & 0.0788 $\pm$ 0.0432 \\
tp\_low & 0.1201 & 1.0000 & 0.0000 & 0.9573 $\pm$ 0.1300 & 0.0511 $\pm$ 0.0360 & 0.0178 $\pm$ 0.0298 \\
\bottomrule
\end{tabular}}
\end{table}

\paragraph{Frequency-stratified metrics.}
Table~\ref{tab:mutag-multiseed-strata} reports the seed-aggregated
Head/Torso/Tail metrics. The seed-mean Tail-$\rho$ ($0.78 \pm 0.53$
at $n=2$) is consistent with the seed-42 value ($0.82$) reported in
Table~\ref{tab:strata}: both lie within between-seed variance, and
the wide CI is dominated by the small sample size ($t_{0.975}(1) =
12.7$ for $n=2$). Head and Torso $\rho$ sit near $0.91$ and $0.97$
with much narrower CI, mirroring the saturation visible in the
main-text single-seed results. A definitive seed-stability claim
awaits the seed-2024 run currently in progress.

\begin{table}[h]
\centering
\caption{MUTAG, DFS canonical, multi-seed strata aggregate (2 seeds:
  \{42, 1337\}; seed 2024 in progress; mean $\pm$ 95\% Student-$t$ CI).}
\label{tab:mutag-multiseed-strata}
\resizebox{\textwidth}{!}{%
\begin{tabular}{lrrrrr}
\toprule
Stratum & Shared & Train-only & Missing & $\rho$ & JSD \\
\midrule
head & 70.0000 & 0.0000 & 0.0000 & 0.9112 & $4.67e-05$ $\pm$ $5.59e-04$ \\
torso & 440.0000 & 0.0000 & 0.0000 & 0.9719 $\pm$ 0.0407 & $8.39e-04$ $\pm$ 0.0050 \\
tail & 91.0000 $\pm$ 368.4799 & 33.0000 $\pm$ 368.4799 & 0.2664 $\pm$ 2.9815 & 0.7788 $\pm$ 0.5289 & 0.1138 $\pm$ 1.3038 \\
\bottomrule
\end{tabular}}
\end{table}

\paragraph{NLL hit/miss analysis.}
Table~\ref{tab:mutag-multiseed-nll} aggregates the per-sequence NLL
hit/miss test of \S\ref{sec:results} (Overall distribution alignment) across seeds. The
seed-mean Mann--Whitney $p = 0.0753 \pm 0.0113$ does not cross the
$0.05$ threshold, so on the seeds available we cannot reject the null
that hit and miss training sequences carry the same NLL---consistent
with the single-seed result and with a memorization-dominated
regime in which the model treats reproduced and non-reproduced
training sequences alike.

\begin{table}[h]
\centering
\caption{MUTAG, DFS canonical, multi-seed NLL analysis (2 seeds:
  \{42, 1337\}; seed 2024 in progress).}
\label{tab:mutag-multiseed-nll}
\begin{tabular}{lrrr}
\toprule
Group & $n$ & Mean NLL & Median NLL \\
\midrule
Hit & 168 $\pm$ 6 & 0.4762 $\pm$ 0.6322 & 0.3506 $\pm$ 0.3347 \\
Miss & 20 $\pm$ 6 & 0.5899 $\pm$ 1.0248 & 0.4814 $\pm$ 1.0688 \\
\midrule
\multicolumn{4}{l}{MW $p$-value: 0.0564 $\pm$ 0.2458} \\
\bottomrule
\end{tabular}
\end{table}

A definitive seed-stability assessment, including the seed=2024
contribution and four-other-TU multi-seed aggregates, is pending
completion of the GPU sweep currently underway and will appear in the
camera-ready revision.

\section{PCQM4Mv2: Detailed Results}
\label{app:pcqm}

\paragraph{LLM-DFS memorization summary.}
Table~\ref{tab:pcqm-mem} reports the headline whole-graph numbers
behind the main-text PCQM4Mv2 scaling result (\S\ref{sec:results}): precision
(exact-match rate) is $0.317$, novelty $0.683$, and the unique rate
saturates at $1.000$, so $1024$ generations from
LLaMA-\textsc{small} produce $1024$ distinct DFS codes of which
roughly one in three matches a training sequence---the regime in
which Spearman~$\rho$ stays at $0.98$ even though most generations
are not verbatim training graphs.

\begin{table}[h]
\centering
\caption{PCQM4Mv2 LLM-DFS memorization summary (single training
  seed = 42; $3{,}104{,}677$ unique training DFS sequences out of
  $3{,}371{,}958$ total; 1024 generated samples).}
\label{tab:pcqm-mem}
\begin{tabular}{lr}
\toprule
Metric & Value \\
\midrule
Train total & 3371958 \\
Train unique & 3104677 \\
Gen total & 1024 \\
Gen unique & 1024 \\
Unique rate & 1.0000 \\
Precision & 0.3174 \\
Novelty & 0.6826 \\
Recall & $1.05e-04$ \\
\bottomrule
\end{tabular}

\end{table}

\paragraph{Scaling memorization across training-set sizes.}
Table~\ref{tab:pcqm-scaling} contrasts the LLaMA-\textsc{small}
exact-match recall and distribution-alignment metrics across the
five TU corpora and PCQM4Mv2 under the same architecture and
\texttt{both\_default} decoding, supporting the headline crossover
of the main-text PCQM4Mv2 scaling (\S\ref{sec:results}).

\begin{table}[h]
\centering
\caption{Scaling memorization across training-set sizes (DFS
  canonical, \texttt{both\_default}). At PCQM4Mv2 scale, memorization
  drops from $\sim$100\% to 31.7\% with the \emph{same} 132M-parameter
  LLaMA-\textsc{small} architecture and token budget shape, while
  Spearman $\rho$ remains high.}
\label{tab:pcqm-scaling}
\setlength{\tabcolsep}{4pt}
\begin{tabular}{lrrrrrr}
\toprule
Dataset & \#Graphs & Memorization & Novelty & Unique & $\rho$ & JSD \\
\midrule
MUTAG & 188 & 0.828 & 0.172 & 0.199 & 0.976 & 0.035 \\
PTC\_MR & 344 & 0.983 & 0.017 & 0.285 & 0.950 & 0.063 \\
ENZYMES & 600 & 0.975 & 0.025 & 0.430 & 0.939 & 0.019 \\
PROTEINS & 1113 & 0.891 & 0.109 & 0.591 & 0.988 & 0.022 \\
NCI1 & 4110 & 1.000 & 0.000 & 0.862 & 0.990 & 0.010 \\
PCQM4Mv2 & 3746620 & 0.317 & 0.683 & 1.000 & 0.976 & 0.044 \\
\bottomrule
\end{tabular}
\end{table}

\paragraph{Subsampled evaluation.}
Table~\ref{tab:pcqm-subsample} reports gSpan-based metrics averaged
over five random $10$k-graph training subsamples (against the same
$1024$ generations) to give a sense of subsample variance.
Spearman~$\rho$ on the train support is $0.976 \pm 0.005$ and on the
support intersection $0.978 \pm 0.005$, with WL-MMD
$1.4 \times 10^{-3}$ and coverage $0.94$---tight bands indicating
that the $\rho \approx 0.98$ headline of the main-text PCQM4Mv2 scaling (\S\ref{sec:results}) is
not an artifact of a single training subsample.

\begin{table}[h]
\centering
\caption{PCQM4Mv2 LLM-DFS subsample evaluation (single training
  seed = 42; 5 random subsamples for reference-side variance
  estimation only, not training- or decoding-seed variance).}
\label{tab:pcqm-subsample}
\begin{tabular}{lr}
\toprule
Metric & mean $\pm$ std \\
\midrule
Shared patterns & 252.0 $\pm$ 3.8 \\
Train-only patterns & 3.2 $\pm$ 0.4 \\
Test-only patterns & 75.0 $\pm$ 3.8 \\
Train mass missing & 0.0065 $\pm$ 0.0008 \\
Test mass novel & 0.1153 $\pm$ 0.0062 \\
JS divergence & 0.0444 $\pm$ 0.0023 \\
Spearman ($\rho$ on train) & 0.9758 $\pm$ 0.0049 \\
Spearman ($\rho$ on $\cap$) & 0.9779 $\pm$ 0.0050 \\
WL MMD & 0.0014 $\pm$ $1.96e-04$ \\
WL coverage & 0.9400 $\pm$ 0.0030 \\
\bottomrule
\end{tabular}

\end{table}

\paragraph{Stratified metrics on PCQM4Mv2 (full).}
Table~\ref{tab:pcqm-strata-full} reports the full Head/Torso/Tail
breakdown for both LLM-DFS and the same-corpus DiGress baseline,
extending the Tail-only summary of Table~\ref{tab:pcqm-strata} in
the main text. Decoding-stage variance is expected to be
substantially smaller on PCQM4Mv2 than on TU: the tail support
contains $\sim 10^4$ patterns rather than $\sim 10^2$, so individual
sampling noise is absorbed across two orders of magnitude more
pattern slots before it can perturb stratum-level summaries
(cf.\ the contrasting large TU std values in
Appendix~\ref{app:decoding-ablation}). A full PCQM4Mv2 decoding-seed
sweep is left for the camera-ready stage.

\begin{table}[h]
\centering
\caption{Stratified metrics on PCQM4Mv2 (single training seed = 42;
  mean $\pm$ std over 5 random 10\,k training-reference subsamples;
  reference-side variance only, not training- or decoding-seed
  variance; c.f.\ caption of Table~\ref{tab:pcqm-strata}). LLM-DFS reproduces
  Head/Torso almost perfectly and only loses ground in the Tail;
  DiGress collapses uniformly with Head missing-mass already at
  $46.5\%$. The contrast separates ``high $\rho$ with non-trivial
  memorization'' (LLM-DFS) from ``low $\rho$ with weak
  labeled-substructure alignment'' (DiGress).}
\label{tab:pcqm-strata-full}
\setlength{\tabcolsep}{4pt}
\begin{tabular}{llrrr}
\toprule
Model & Stratum & Missing & JSD & $\rho$ \\
\midrule
\multirow{3}{*}{LLM-DFS}
  & Head  & 0.0000 $\pm$ 0.0000 & $3.07\!\times\!10^{-4}$ $\pm$ $4.65\!\times\!10^{-5}$ & 0.9834 $\pm$ 0.0043 \\
  & Torso & 0.0045 $\pm$ 0.0015 & 0.0022 $\pm$ $5.28\!\times\!10^{-4}$ & 0.9630 $\pm$ 0.0087 \\
  & Tail  & 0.0694 $\pm$ 0.0170 & 0.0255 $\pm$ 0.0062 & 0.4061 $\pm$ 0.2223 \\
\midrule
\multirow{3}{*}{DiGress}
  & Head  & 0.4650 $\pm$ 0.0527 & 0.2007 $\pm$ 0.0280 & 0.5385 $\pm$ 0.0330 \\
  & Torso & 0.7765 $\pm$ 0.0331 & 0.4117 $\pm$ 0.0290 & 0.1395 $\pm$ 0.0333 \\
  & Tail  & 0.8600 $\pm$ 0.0328 & 0.4887 $\pm$ 0.0288 & 0.0463 $\pm$ 0.2115 \\
\bottomrule
\end{tabular}

\end{table}

\paragraph{Capacity sweep at fixed corpus (DFS canonical, full).}
Table~\ref{tab:pcqm-model-scaling-app} reports the four-point capacity
sweep summarized in Figure~\ref{fig:pcqm-scaling}. Subgraph-level
alignment ($\rho$, JSD) and whole-graph similarity (WL-MMD) stay in
narrow bands across the $43\times$ parameter range. Tail rank correlation improves with
capacity, reaching its best value at \textsc{large}
(Tail $\rho \approx 0.43$); the \textsc{medium} point ($0.34$) sits
within two standard deviations of \textsc{tiny} and does not break
the upward trend. memorization and missing mass instead track
tokens-per-parameter: \textsc{small} (128 tok/param) reaches the
highest memorization ($0.32$) and the lowest Tail missing mass
($0.07$), while the under-trained \textsc{large} (9 tok/param,
12 epochs) memorizes only $0.20$ despite $3.4\times$ more parameters.
Exact-match recall is therefore bottlenecked by gradient updates
rather than parameter budget; the Tail-stratum gains we report for
\textsc{large} are likely a lower bound under additional training.
We treat this as a diagnostic capacity study rather than a
scaling-law fit.
\begin{table}[h]
\centering
\caption{Model-size scaling on PCQM4Mv2 (DFS canonical,
  \texttt{both\_default}, single training seed = 42; mean $\pm$
  half-width over 5 random 10\,k-graph training-reference subsamples
  for a single $1024$-graph generation; reference-side variance only,
  not training- or decoding-seed variance; c.f.\ caption of
  Table~\ref{tab:pcqm-strata}; memorization on the full 3.37M unique
  training set). Tok/Param = total training tokens divided by
  parameter count; \textsc{tiny}/\textsc{medium}/\textsc{small} are
  trained for 50 epochs and \textsc{large} for 12 epochs (compute
  budget).}
\label{tab:pcqm-model-scaling-app}
\resizebox{\textwidth}{!}{%
\begin{tabular}{lrrrrrrr}
\toprule
Model & Params & Tok/Param & Memorization & $\rho$ & Tail $\rho$ & Tail miss & WL-MMD \\
\midrule
LLaMA-tiny   & 10.5\,M  & 1600 & 0.126 & 0.9644 & 0.3740 & 0.1768 & 0.0023 \\
LLaMA-medium & 99.1\,M  & 170  & 0.258 & 0.9482 & 0.3355 & 0.1915 & 0.0016 \\
LLaMA-small  & 132.2\,M & 128  & 0.317 & 0.9758 & 0.4061 & 0.0694 & 0.0014 \\
LLaMA-large  & 451.0\,M & 9$^{\dagger}$    & 0.198 & 0.9686 & 0.4286 & 0.1614 & 0.0023 \\
\bottomrule
\multicolumn{8}{l}{\footnotesize $^{\dagger}$ Chinchilla-suboptimal: trained for 12 epochs ($\sim$4B tokens) due to compute budget.} \\
\end{tabular}
}
\end{table}
\begin{figure}[h]
  \centering
  \includegraphics[width=0.98\linewidth]{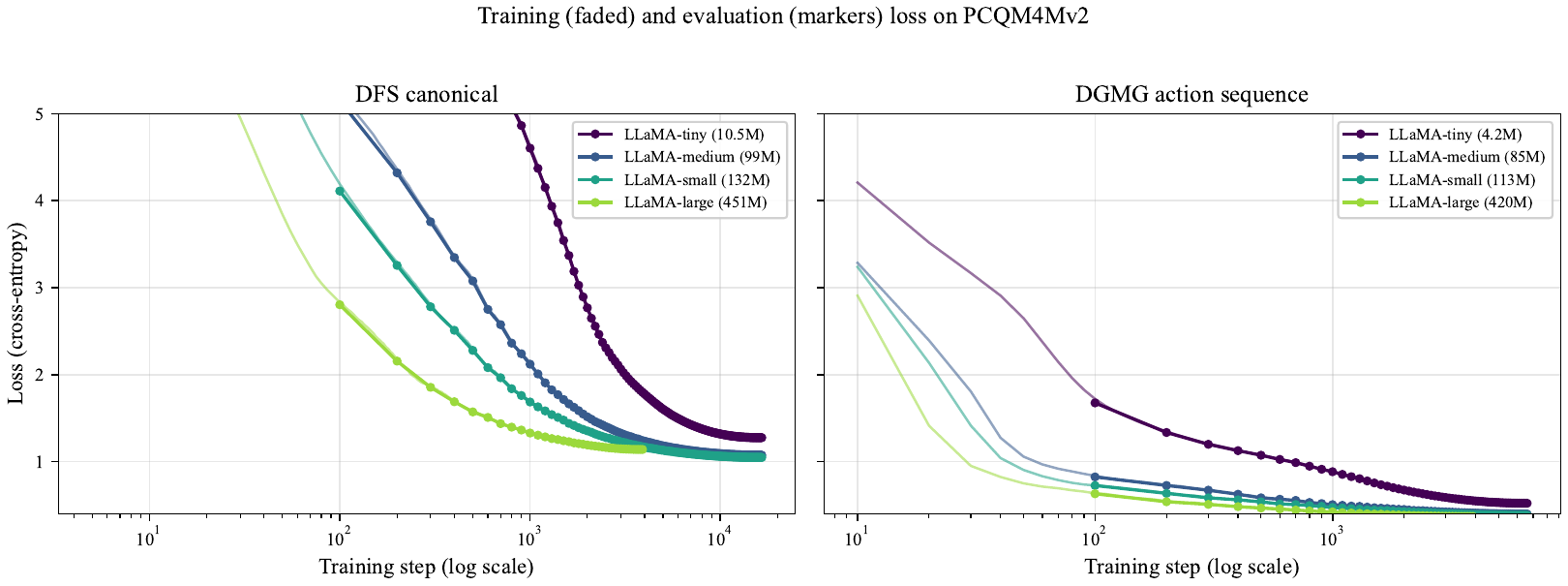}
  \caption{Training (faded curves) and held-out evaluation (markers)
    cross-entropy loss on PCQM4Mv2 across model sizes; training step
    on a log axis to match
    Figure~\ref{fig:pcqm-tiny-trajectory}. Left: DFS canonical
    (\textsc{tiny}/\textsc{medium}/\textsc{small} share a 50-epoch
    schedule, \textsc{large} a 12-epoch schedule). Right: DGMG action
    sequence
    (\textsc{tiny}/\textsc{medium}/\textsc{small} share a 20-epoch
    schedule, \textsc{large} a 12-epoch schedule). Loss separation by
    capacity is consistent across both serializations: larger models
    reach lower train and eval loss in every regime we examine.}
  \label{fig:pcqm-loss-curves}
\end{figure}

\paragraph{Checkpoint trajectory.}
Figure~\ref{fig:pcqm-tiny-trajectory} traces unique-rate, precision,
and Spearman~$\rho$ across PCQM4Mv2 training for the four LLaMA
capacities. Unlike the small-corpus MUTAG dynamics
(Figure~\ref{fig:learning-curves}), unique rate stays at $1.0$
throughout training and precision climbs only mildly with steps
while increasing with capacity at fixed step budget---i.e.\ at this
scale the model never collapses onto the training corpus, and
larger capacity buys a larger memorized fraction rather than
faster memorization.

\begin{figure}[h]
  \centering
  \includegraphics[width=0.9\linewidth]{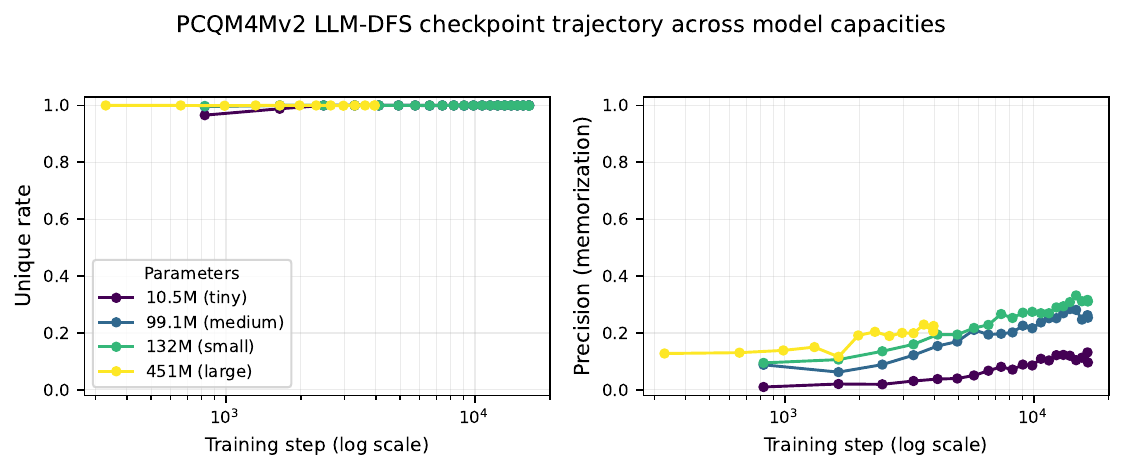}
  \caption{PCQM4Mv2 LLM-DFS checkpoint trajectory across the four
    LLaMA capacities sweep (10.5M, 99.1M, 132M, 451M parameters;
    legend ordered by parameter count, not by the upstream
    \textsc{tiny}/\textsc{medium}/\textsc{small}/\textsc{large} naming
    which is non-monotonic in size).
    Training step on a log axis. Unique rate saturates at $1.0$ within
    the first $10^3$ steps and stays there for every capacity, while
    precision (memorization) rises only mildly with training and
    increases with capacity at fixed step budget---in sharp contrast
    to the small-corpus MUTAG dynamics in
    Figure~\ref{fig:learning-curves}, where precision climbs to $0.8$
    over the same step range.}
  \label{fig:pcqm-tiny-trajectory}
\end{figure}

\paragraph{DGMG capacity sweep.}
A three-point sweep on DGMG (identical 20-epoch schedule for all
three sizes) reproduces the DFS pattern and additionally separates
two regimes. Subgraph-level alignment ($\rho$, JSD) and whole-graph
similarity (WL-MMD) improve sharply from \textsc{tiny} ($4.2$M) to
\textsc{medium} (85M)
($\rho: 0.89 \!\to\! 0.95$, JSD: $0.29 \!\to\! 0.25$, WL-MMD:
$0.0037 \!\to\! 0.0021$) and saturate between \textsc{medium} and
\textsc{small} (113M; $\rho$ shifts by $<\!0.005$, WL-MMD by $10\%$).
Missing-mass metrics keep improving across all three sizes: Tail
missing mass drops $0.56 \!\to\! 0.20 \!\to\! 0.13$ ($76\%$ overall)
and train missing mass drops $0.09 \!\to\! 0.02 \!\to\! 0.02$.
memorization scales monotonically with capacity
($0.007 \!\to\! 0.060 \!\to\! 0.085$), mirroring the DFS pattern when
training duration is held constant. Because all three DGMG variants
share the schedule, the capacity effect is uncontaminated by
duration differences and provides a cleaner read on the
scaling dimension that the DFS sweep partially confounds with
training length. The headline trends---monotonic capacity gains on rank
correlation; diminishing returns on distributional alignment beyond
$\sim$100M parameters---generalize across the two serializations.

\paragraph{Extended capacity check (DGMG-large).}
We add a fourth, larger point at \textsc{large} (419.7M parameters),
trained for $12$ epochs under a fixed compute budget so its
tokens-per-parameter ratio (5.8) sits well below the other three sizes
(960, 47.9, 35.9 for \textsc{tiny}/\textsc{medium}/\textsc{small}).
Despite this under-training, \textsc{large} matches or exceeds all
three smaller variants on every subgraph-level alignment metric
($\rho = 0.964$, JSD $= 0.052$) and whole-graph similarity
(WL-MMD $= 0.0020$), continues the
monotonic Tail $\rho$ improvement ($0.158 \!\to\! 0.241 \!\to\! 0.246
\!\to\! 0.279$), and pushes Tail missing mass to its lowest value
($0.093$, vs.\ $0.132$ at \textsc{small}). memorization also keeps
rising with capacity ($0.085 \!\to\! 0.115$) even at $5.8$ tokens per
parameter, indicating that the head of the distribution is recoverable
with a fraction of the optimal training budget. The four-point sweep
therefore preserves both the saturation of subgraph-level alignment
and whole-graph similarity around 100M parameters and the monotonic
Tail-stratum gains
with capacity that we observe in the DFS sweep
(Table~\ref{tab:pcqm-model-scaling-app}).

\begin{table}[h]
\centering
\caption{Model-size scaling on PCQM4Mv2 (DGMG action sequences,
  single training seed = 42;
  20-epoch schedule for \textsc{tiny}/\textsc{medium}/\textsc{small},
  12-epoch budget-limited schedule for \textsc{large}, marked
  $\dagger$; mean $\pm$ half-width over 5 random 10\,k-graph
  reference subsamples; reference-side variance only, not training-
  or decoding-seed variance; memorization on the full $3.10$M unique
  training DFS sequences). Parameter counts are smaller than the DFS counterparts
  because the DGMG vocabulary is $\sim 86$ tokens versus 12{,}297 for
  DFS.}
\label{tab:pcqm-dgmg-model-scaling}
\begin{tabular}{lrrrrrr}
\toprule
Model & Params & Memorization & $\rho$ & Tail $\rho$ & Tail miss & WL-MMD \\
\midrule
LLaMA-tiny    & 4.2\,M    & 0.007 & 0.8890 & 0.1579 & 0.5555 & 0.0037 \\
LLaMA-medium  & 85.1\,M   & 0.060 & 0.9481 & 0.2412 & 0.2009 & 0.0021 \\
LLaMA-small   & 113.4\,M  & 0.085 & 0.9452 & 0.2457 & 0.1317 & 0.0023 \\
LLaMA-large$^\dagger$ & 419.7\,M & 0.115 & 0.9635 & 0.2790 & 0.0928 & 0.0020 \\
\bottomrule
\multicolumn{7}{l}{$^\dagger$ Trained for 12 epochs ($\sim$2.4\,B tokens, 5.8 tokens/parameter)}\\
\multicolumn{7}{l}{vs.\ 20 epochs ($\sim$4.1\,B tokens) for the other three sizes.}\\
\end{tabular}

\end{table}

\paragraph{Novel-only subset analysis.}
\label{app:pcqm-novel-only}
A natural concern is that the high $\rho$ on PCQM4Mv2 might still be
driven by the non-trivial memorized fraction. To test this directly,
we recompute gSpan on the novel-only subset
$\mathbf{D}_{\mathrm{gen}} \setminus \mathbf{D}_{\mathrm{tr}}$
($679$, $730$, $797$ graphs out of $1024$ for
\textsc{small}/\textsc{medium}/\textsc{large}, novelty
$66$--$78\%$) against a fixed $10{,}000$-graph training subsample
with the same gSpan settings ($\sigma = 0.1$, $\mathrm{upper} = 8$).
Table~\ref{tab:novel-only-pcqm} shows that the novel-only Spearman
$\rho$ tracks the all-generation reference within $0.026$ at every
capacity (per-row gaps: \textsc{small} $0.004$,
\textsc{medium} $0.026$, \textsc{large} $0.014$), even though the
novel subset by construction excludes the memorized graphs. The
distributional alignment is therefore not driven solely by the
memorized fraction.

\begin{table}[h]
\centering
\caption{PCQM4Mv2 novel-only gSpan against a fixed $10{,}000$-graph
  training subsample ($\sigma = 0.1$, $\mathrm{upper} = 8$; single
  training seed = 42, single $1024$-graph generation draw).
  ``All-gen $\rho$'' is the Spearman on the full $1024$-graph
  generation, computed under the same single-subsample reference for
  fair comparison and therefore differs from the $5$-subsample
  averages in Tables~\ref{tab:pcqm-scaling}--\ref{tab:pcqm-strata};
  the salient quantity is the small all-vs-novel gap.}
\label{tab:novel-only-pcqm}
\begin{tabular}{lrrrrrrr}
\toprule
& & & & \multicolumn{3}{c}{Novel-only subset} \\
\cmidrule(lr){5-7}
Size & $n_\text{gen}$ & Novelty & All-gen $\rho$ & $\rho$ & JSD & MM \\
\midrule
\textsc{small} & 1024 & 0.663 & 0.793 & 0.789 & 0.136 & 0.130 \\
\textsc{medium} & 1024 & 0.713 & 0.798 & 0.824 & 0.144 & 0.147 \\
\textsc{large} & 1024 & 0.778 & 0.822 & 0.808 & 0.146 & 0.134 \\
\bottomrule
\end{tabular}

\end{table}

\paragraph{DiGress comparator.}
Tables~\ref{tab:pcqm-digress-mem}--\ref{tab:pcqm-digress-strata}
report the DiGress baseline at PCQM4Mv2 scale under the same
diagnostic. DiGress occupies the opposite corner of the
memorization--alignment plane to the LLM: precision is
$0.000$ (no exact-match recall over $200$ generated graphs) and
distributional alignment collapses (Spearman on the train
support $0.247 \pm 0.051$, train missing mass $0.699 \pm 0.031$,
JSD $0.401 \pm 0.031$), with the Tail stratum showing the same
pattern. The contrast positions the LLM regime as
\emph{distinct} from low-memorization diffusion---high $\rho$ with
moderate memorization rather than low $\rho$ with near-zero
memorization.

\begin{table}[h]
\centering
\caption{PCQM4Mv2 DiGress memorization summary (DiGress open-source
  checkpoint; comparison LLM numbers use single training seed = 42).}
\label{tab:pcqm-digress-mem}
\begin{tabular}{lr}
\toprule
Metric & Value \\
\midrule
Train total & 3371958 \\
Train unique & 3104677 \\
Gen total & 200 \\
Gen unique & 200 \\
Unique rate & 1.0000 \\
Precision & 0.0000 \\
Novelty & 1.0000 \\
Recall & 0.0000 \\
\bottomrule
\end{tabular}
\end{table}

\begin{table}[h]
\centering
\caption{PCQM4Mv2 DiGress subsample evaluation (5 random reference
  subsamples, reference-side variance only).}
\label{tab:pcqm-digress-subsample}
\begin{tabular}{lr}
\toprule
Metric & mean $\pm$ std \\
\midrule
Shared patterns & 58.2 $\pm$ 8.1 \\
Train-only patterns & 197.0 $\pm$ 6.4 \\
Test-only patterns & 72.8 $\pm$ 8.1 \\
Train mass missing & 0.6988 $\pm$ 0.0308 \\
Test mass novel & 0.3326 $\pm$ 0.0605 \\
JS divergence & 0.4006 $\pm$ 0.0307 \\
Spearman ($\rho$ on train) & 0.2473 $\pm$ 0.0512 \\
Spearman ($\rho$ on $\cap$) & 0.7008 $\pm$ 0.0658 \\
WL MMD & 0.0237 $\pm$ 0.0146 \\
WL coverage & 0.9034 $\pm$ 0.0194 \\
\bottomrule
\end{tabular}
\end{table}

\begin{table}[h]
\centering
\caption{PCQM4Mv2 DiGress strata (5 random reference subsamples,
  reference-side variance only).}
\label{tab:pcqm-digress-strata}
\begin{tabular}{lrrr}
\toprule
Stratum & Missing & JSD & $\rho$ \\
\midrule
Head & 0.4650 $\pm$ 0.0527 & 0.2007 $\pm$ 0.0280 & 0.5385 $\pm$ 0.0330 \\
Torso & 0.7765 $\pm$ 0.0331 & 0.4117 $\pm$ 0.0290 & 0.1395 $\pm$ 0.0333 \\
Tail & 0.8600 $\pm$ 0.0328 & 0.4887 $\pm$ 0.0288 & 0.0463 $\pm$ 0.2115 \\
\bottomrule
\end{tabular}
\end{table}

\section{Learning Dynamics}
\label{app:learning-dynamics}

Figure~\ref{fig:learning-curves} shows the evolution of key metrics
across training checkpoints.

\begin{figure}[h]
  \centering
  \includegraphics[width=0.95\linewidth]{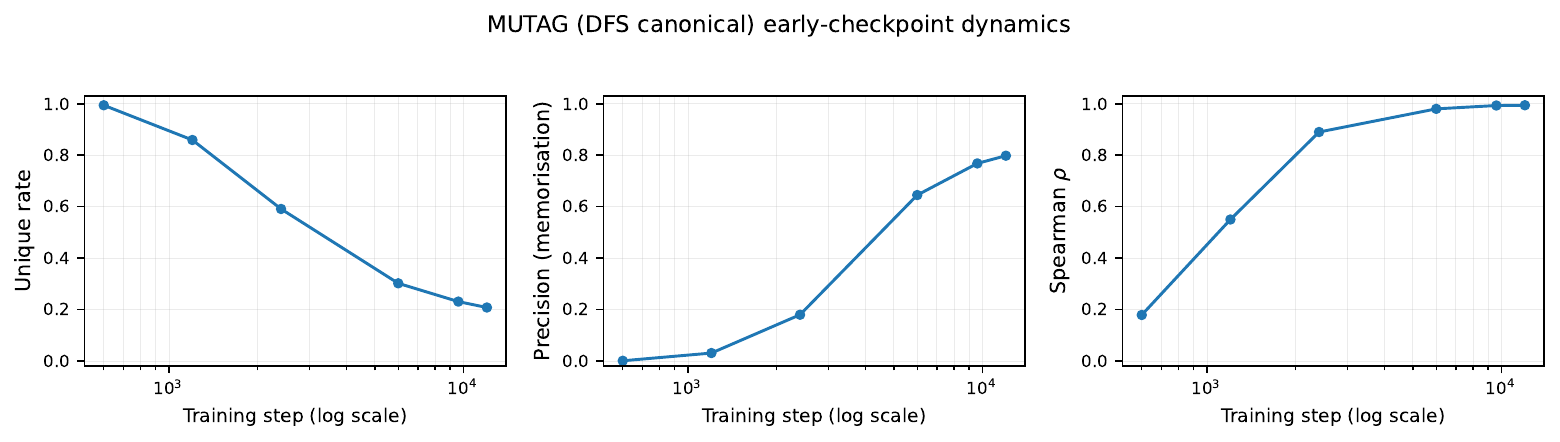}
  \caption{MUTAG (DFS canonical) early-checkpoint dynamics. Unique
    rate, precision, and Spearman $\rho$ are tracked across training
    checkpoints. Spearman $\rho$ rises in lockstep with precision,
    indicating that distributional alignment is acquired
    \emph{simultaneously} with memorization rather than after it.
    Per-dataset checkpoint trajectories for the remaining four TU
    datasets are reported in Figure~\ref{fig:trajectories-all}.}
  \label{fig:learning-curves}
\end{figure}

A consistent three-phase pattern emerges across the five TU datasets
(see Figure~\ref{fig:trajectories-all} for the per-dataset
trajectories):
(1)~a \emph{random generation phase} early in training
($\rho < 0.5$, precision $\approx 0\%$, novelty $\approx 100\%$);
(2)~a \emph{rapid memorization phase} where $\rho$ rises sharply from
$\sim$0.5 to $>$0.9 in lockstep with precision; and
(3)~a \emph{saturation phase} ($\rho > 0.94$, precision $80$--$100\%$)
in which further training yields marginal improvement. Critically, the
rise in $\rho$ is \emph{synchronized} with the rise in memorization
rate, not delayed. The MUTAG DFS trajectory illustrates this: at step
600, precision is $0.1\%$ and $\rho = 0.18$; by step 2400, precision
reaches $18\%$ and $\rho = 0.89$; at the final checkpoint
(step $\sim$12000), precision is $79.8\%$ and $\rho = 0.99$.

\end{document}